\documentclass{article}

\usepackage{microtype}
\usepackage{graphicx}
\usepackage{subcaption}
\usepackage{booktabs} 

\usepackage{hyperref}

\usepackage[accepted]{icml2026}

\usepackage{amsmath}
\usepackage{amssymb}
\usepackage{mathtools}
\usepackage{amsthm}

\usepackage[capitalize,noabbrev]{cleveref}

\usepackage{xspace}
\usepackage{pifont}
\usepackage{booktabs}       
\usepackage{amsfonts}       
\usepackage{microtype}      
\usepackage{xcolor}         
\usepackage{graphicx}
\usepackage{subcaption}
\usepackage{thmtools}
\usepackage{booktabs}
\usepackage{adjustbox}
\usepackage{wrapfig}
\usepackage[normalem]{ulem}
\usepackage{mathrsfs}
\usepackage{multirow}
\usepackage{svg}
\usepackage{thm-restate}
\usepackage{tabularx}
\usepackage{colortbl}
\usepackage{makecell}
\usepackage{titletoc}

\usepackage{amsmath,amsfonts,bm}

\def\eqref#1{equation~\ref{#1}}

\def\1{\bm{1}}

\DeclareMathAlphabet{\mathsfit}{\encodingdefault}{\sfdefault}{m}{sl}
\SetMathAlphabet{\mathsfit}{bold}{\encodingdefault}{\sfdefault}{bx}{n}

\theoremstyle{plain}
\newtheorem{theorem}{Theorem}[section]

\theoremstyle{definition}
\newtheorem{definition}[theorem]{Definition}

\theoremstyle{remark}

\newcommand{\ie}{i.e.\xspace}
\newcommand{\eg}{e.g.\xspace}

\usepackage[textsize=tiny]{todonotes}

\icmltitlerunning{Efficient Continual Alignment for Open-Ended Image-to-Text Generation}

\begin{document}

\twocolumn[
  \icmltitle{ECA: Efficient Continual Alignment for Open-Ended Image-to-Text Generation}



  \icmlsetsymbol{equal}{*}

  \begin{icmlauthorlist}
    \icmlauthor{Jiangtao Kong}{wm}
    \icmlauthor{Peijun Zhao}{JP}
    \icmlauthor{Chun-Fu Chen}{JP}
    \icmlauthor{Youngwook Do}{JP}
    \icmlauthor{Shaohan Hu}{JP}
    \icmlauthor{Tianyi Zhou}{mbzuai}
    \icmlauthor{Huajie Shao}{wm}
  \end{icmlauthorlist}

  \icmlaffiliation{wm}{Department of Computer Science, William \& Mary, Virginia, USA}
  \icmlaffiliation{mbzuai}{Mohamed bin Zayed University of Artificial Intelligence, Masdar City, Abu Dhabi}
  \icmlaffiliation{JP}{Global Technology Applied Research, JPMorganChase, USA}

  \icmlcorrespondingauthor{Huajie Shao}{hshao@wm.edu}

  \icmlkeywords{Machine Learning, ICML}

  \vskip 0.3in
]



\printAffiliationsAndNotice{}  

\begin{abstract}
Incremental Learning (IL) for Open-ended Image-to-Text Generation (OpenITG) enables models to continuously generate accurate, contextually relevant text for new images while preserving previously acquired knowledge. Unlike prior studies, this paper addresses a more practical scenario in which the predominant category of visual data shifts over time as environments evolve. In this context, we introduce a new notion of continual alignment, which incrementally adapts the alignment module within pre-trained VLMs to preserve high-quality cross-modal representations. Based on this idea, we propose \textbf{E}fficient \textbf{C}ontinual \textbf{A}lignment (ECA), a novel exemplar-free IL approach for OpenITG. The key challenge is enabling the model to acquire new, task-specific features while minimizing interference with the established alignment without accessing raw data from previous tasks.
To address this, ECA employs three core mechanisms: a \textbf{M}ixture \textbf{o}f \textbf{Q}uery (MoQ) module that adapts task-specific query tokens, a \textbf{F}ish\textbf{e}r \textbf{D}ynamic \textbf{Ex}pansion (FeDEx) that dynamically expands model structure based on a Fisher Information Matrix (FIM)-based metric, and an embedding dictionary with \textbf{D}ictionary \textbf{R}eplay (DR) to retain past knowledge. To evaluate ECA's performance, we construct four new IL OpenITG benchmarks that better reflect real-world scenarios. Experimental results demonstrate that ECA significantly mitigates catastrophic forgetting and improves IL performance compared to baseline methods. Code and benchmarks are available at \url{https://github.com/Snowball0823/ECA}.


\end{abstract}
 
\section{Introduction}\label{sec:intro}
Open-ended Image-to-Text Generation (OpenITG) tasks, such as image captioning~\citep{vinyals2015show,Xu2015ShowAA,herdade2019image,Ramos_2023_CVPR} and open-ended Visual Question Answering (VQA)~\citep{antol2015vqa,xu-etal-2020-open,fu-etal-2023-generate}, require Vision-Language Models (VLMs) to generate accurate, contextually relevant text based on given images. In real-world scenarios, the visual content shifts as environments and time evolve, leading to significant distribution changes. 
This dynamic setting motivates Incremental Learning (IL) for OpenITG, where a model adapts to evolving visual streams while maintaining generation quality.
\begin{figure}[!t]
  \vspace{-0.05in}
  \centering
  \includegraphics[width=\linewidth]{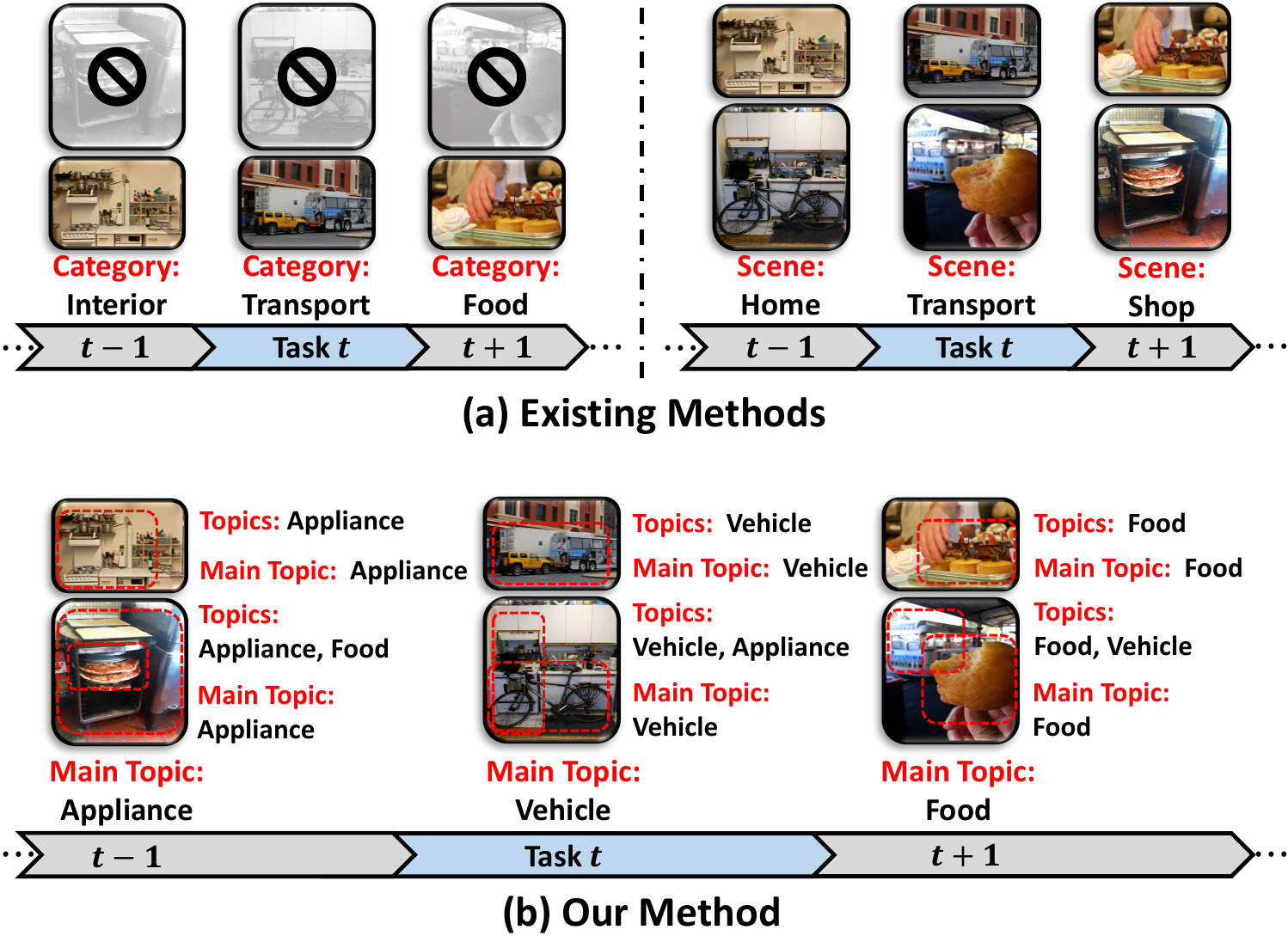}
  \vspace{-0.1in}
  \caption{Comparison between (a) existing task splits~\citep{del2020ratt,zhang2023vqacl,lei2023symbolic} and (b) our main topic split. \textbf{(a) top-left} illustrates methods that assume disjoint object categories and discard images containing multiple topics. \textbf{(a) top-right} illustrates methods that rely on disjoint background scenes. \textbf{(b)} defines each task by the image’s dominant semantic category (``main topic''), which accommodates overlapping semantics and shifts in focus across time or environments, yielding a more realistic continual OpenITG setting.
  }
  \label{fig:intro}
  \vspace{-0.1in}
\end{figure}
As shown in Fig.~\ref{fig:intro}(a), existing works~\citep{del2020ratt,zhang2023vqacl,lei2023symbolic} often assume disjoint image categories or background scenes when splitting tasks, then train the model sequentially. 
However, this assumption does not always hold in practice because scenes contain multiple semantic elements whose prominence changes over time. For example, an indoor image dominated by ``appliance'' cues may later be dominated by ``vehicle'' cues, as illustrated in Fig.~\ref{fig:intro}(b). Similarly, even if ``vehicle'' remains the dominant feature, the environment may introduce additional context.
To capture these dynamic variations, we define an image’s \textbf{\emph{main topic}} as the semantic category of its most prominent object. 
Following the definition, we then split tasks by main topic to better reflect the continuous shifts of visual content.
In our work, \emph{we focus on IL for OpenITG tasks by adapting to the ever-changing main topics in visual data with overlapping semantics}.

Under our main-topic setting for IL in OpenITG, the key challenge is to maintain cross-modal alignment while countering catastrophic forgetting~\citep{mccloskey1989catastrophic} and interference from semantic overlap across tasks. As learning proceeds sequentially, the VLM tends to overwrite previously formed associations, which reduces the relevance and accuracy of generated text. Most prior methods mitigate forgetting by fine-tuning fusion and language components sequentially with stored raw exemplars~\citep{del2020ratt,lei2023symbolic,zhang2023vqacl}. However, these approaches introduce three more major drawbacks. First, full-scale fine-tuning for large task-agnostic pre-trained models~\citep{yuan2021tokens,Zhang2022OPTOP,chung2024scaling} is inefficient and can erode pre-training gains~\citep{zhai2023investigating}. Second, storing raw exemplars raises privacy and memory concerns. Moreover, since these methods are based on the assumption of disjoint distributions, they do not explicitly address semantic overlap across tasks, which is closer to real-world settings.
Motivated by these limitations, we first introduce a new notion of \emph{continual alignment}, aiming to achieve the continual adaptation of the alignment module, which establishes the cross-modal alignment within pre-trained VLMs, to preserve high-quality cross-modal representations during sequential task learning. 
In our main topic setting, achieving continual alignment without saving raw exemplars necessitates addressing three challenges: \textbf{C1:} recurring semantics appear without task identifiers, which calls for the compositional reuse of earlier cues. \textbf{C2:} preserving established cross-modal alignment without saving raw exemplars is needed under distribution shift. \textbf{C3:} semantic overlap across tasks can trigger parameter conflict, which must be mitigated during adaptation.
To tackle these challenges, we propose \textbf{E}fficient \textbf{C}ontinual \textbf{A}lignment (ECA), an exemplar-free framework operating at the alignment module. For \textbf{C1}, we introduce the Mixture of Query (MoQ) module, which learns task-specific query tokens and composes via attention to acquire new cues with minimal disruption to prior alignment. For \textbf{C2}, we design the Dictionary Replay (DR) module, which maintains a compact embedding dictionary to capture task-agnostic visual components and replays them effectively without raw exemplars. To address \textbf{C3}, we propose the Fisher Dynamic Expansion (FeDEx), which computes a Fisher Information Matrix (FIM)–based metric, and selectively expands parallel adapters only when interference is detected, thereby preserving established alignment while allocating capacity to new topics. With the interplay of these three modules, ECA significantly alleviates catastrophic forgetting and enhances IL performance in OpenITG tasks.
In this paper, we use BLIP-2~\citep{li2023blip} as a representative model to evaluate the ECA. BLIP-2 exposes the alignment module as a Q-Former, which connects a frozen visual encoder and a frozen language model, enabling us to isolate alignment module behavior and study continual alignment in a controlled evaluation setting. The framework design remains general and can be employed to projector-based MLLMs such as LLaVA~\cite{liu2023visual}. We provide a description of this design in the Appendix~\ref{app:eca_llava}.

To study the continual alignment and evaluate ECA, we build four IL benchmarks for OpenITG based on different main topics. We name these benchmarks as Topic of Semantic for COCO Caption (ToS-COCO Caption), ToS-VQAv2, ToS-TextCaps, and ToS-TextVQA, derived from MSCOCO ImageCaption~\citep{lin2014microsoft}, VQAv2~\citep{goyal2017making}, TextCaps~\citep{sidorov2019textcaps}, and TextVQA~\citep{singh2019towards}. These benchmarks cover image captioning and open-ended VQA. In our setting, models generate text across tasks without task-specific IDs.

In summary, \textbf{our contributions include:}
1) we propose ECA, a novel exemplar-free IL approach for OpenITG that updates only the alignment module while keeping heavy backbones frozen.
To the best of our knowledge, we are the first to explicitly target preserving the cross-modal alignment of the alignment module in pre-trained VLMs during exemplar-free incremental learning for OpenITG.
2) we propose the Fisher Dynamic Expansion (FeDEx) with the Mixture of Query (MoQ) approach. These two modules adapt alignment module capacity and token composition to acquire task-specific features while preserving previously learned knowledge.
By doing so, they continually adapt the alignment module in the pre-trained VLMs to maintain the cross-modal alignment.
3) we design a novel memory mechanism, Dictionary Replay (DR)
based on sparse dictionary learning for IL in OpenITG tasks instead of saving raw exemplars.
4) we construct four benchmarks \ie ToS-COCO Caption, ToS-VQAv2, ToS-TextCaps, ToS-TextVQA, that closely mimic realistic scenarios by defining tasks according to image topics; and 5) Extensive experiments on our new benchmarks show that the proposed ECA achieves superior performance on IL for OpenITG tasks.


\vspace{-0.05in}
\section{Related Work}\label{sec:relatedwork}
\noindent \textbf{Pre-trained Vision-Language Models.} 
Traditional VLMs for OpenITG perform full-scale end-to-end training~\citep{herdade2019image,li2021align,li2022blip,wang2022end}. However, with the rise of large-scale and task-agnostic pre-trained uni-modal models~\citep{dosovitskiy2020image, brown2020language,chung2024scaling}, full fine-tuning becomes inefficient and inflexible. Recent VLMs~\citep{alayrac2022flamingo,li2023blip,liu2023visual,bai2025qwen2} adopt alignment modules to bridge frozen visual encoders and frozen Large Language Models (LLMs), realized as a Q-Former in BLIP-2–style models or as projector tokens in projector-based Multi-modal LLMs, which aligns visual features to the LLM token space and conditions the LLM on the image. However, the alignment module is sensitive to data distribution shifts and prone to catastrophic forgetting in IL scenarios~\citep{zhao2024aligngpt}. Thus, we introduce the notion of \emph{continual alignment} and then propose ECA to effectively improve the continual alignment ability of VLMs for OpenITG.

\noindent \textbf{Incremental Learning.}
Incremental Learning (IL) aims to enable models to acquire new knowledge while preserving previous knowledge. Traditional IL methods generally fall into three categories: (i) regularization-based approaches~\citep{li2017learning,kirkpatrick2017overcoming,aljundi2018memory,ahn2019uncertainty}; (ii) rehearsal-based methods~\citep{rebuffi2017icarl,douillard2020podnet,yan2021dynamically,zhu2021prototype,kong2025yooop}; and (iii) architectural-based methods~\citep{fernando2017pathnet,mallya2018packnet,douillard2022dytox,wang2022foster,wang2022beef}. Recent exemplar-free IL methods based on prompt-tuning~\citep{wang2022s,wang2022dualprompt,wang2022learning,smith2023coda} have demonstrated strong performance on uni-modal tasks such as image classification, yet their effectiveness on multi-modal tasks like OpenITG remains under-explored.
Recent IL approaches targeting OpenITG (captioning/VQA) tasks~\citep{del2020ratt,lei2023symbolic,zhang2023vqacl,chen2025coin} still face limitations. 
For example, VQACL~\citep{zhang2023vqacl} combines prototype learning with exemplar buffers, posing privacy and memory issues. Moreover, these methods typically involve extensive fine-tuning of fusion and language components, which is inefficient for large-scale pre-trained models.  
More recently,
multimodal continual instruction-tuning works adopt PEFT to adapt MLLMs to evolving \emph{textual instructions}~\citep{chen2025coin,cao2024continual,zeng2025modalprompt,guo2025hide}. For example, Continual LLaVA uses a low-rank embedding pool, CoIN and HiDe-LLaVA build on MoE-LoRA variants, and ModalPrompt designs modality-guided prompt tuning. These methods primarily focus on \emph{textual instructions} drift rather than shifts in \textit{visual topics}.
In contrast, our ECA approach effectively addresses these limitations by continuously adapting the alignment module of pre-trained VLMs without raw exemplar buffers or task-specific identifiers, explicitly targeting IL scenarios involving continuous shifts in visual semantic topics.

\noindent \textbf{IL Benchmarks for OpenITG.}
There are several IL benchmarks for OpenITG, including image captioning~\citep{del2020ratt} and open-ended VQA~\citep{zhang2023vqacl,greco2019psycholinguistics}. In these works, tasks are typically split based on disjoint image categories. However, in real-world scenarios, a single image often contains multiple objects, and its overall semantics are better characterized by its dominant visual topic. Thus, the assumption of image distributions across tasks as disjoint is unrealistic. While similar work~\citep{lei2023symbolic} splits tasks based on different scenes, focusing on the background, this approach fails to capture continuous shifts in dominant visual topics. In our work, we provide a new setting based on dominant visual topics and create four IL benchmarks for image captioning and open-ended VQA.

\begin{figure*}[!t]
\vspace{-0.1in}
\centering
\includegraphics[width=0.99\linewidth]{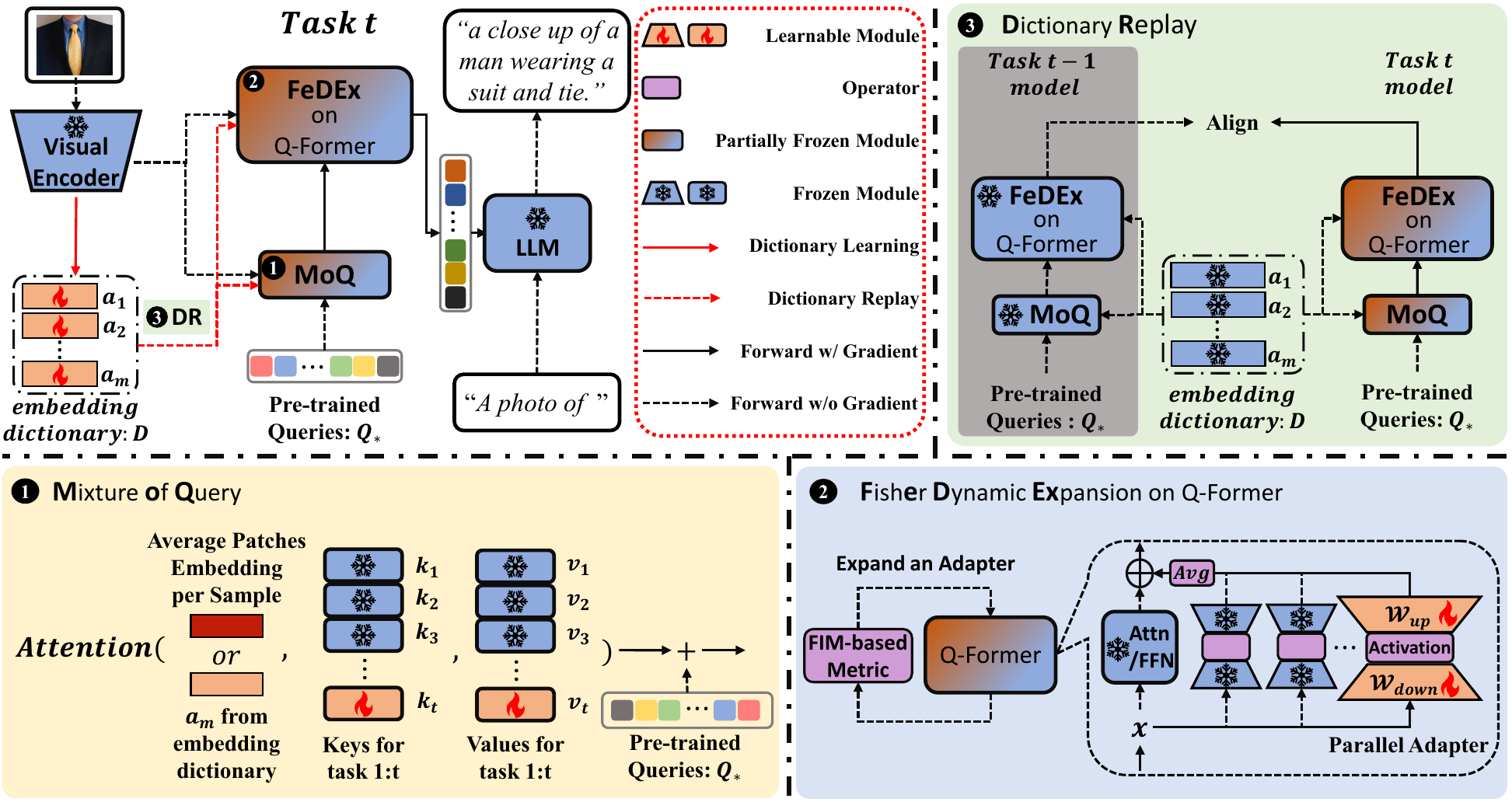}
\vspace{-0.05in}
\caption{The framework of our exemplar-free incremental learning approach, ECA, for image-to-text generation. \textbf{Upper Left:} An input image is processed by a frozen visual encoder to produce features. These features enter the \textbf{M}ixture \textbf{o}f \textbf{Q}uery module (\ding{182}), which generates query tokens to interact with the Q-Former equipped with \textbf{F}ish\textbf{e}r \textbf{D}ynamic \textbf{Ex}pansion (\ding{183}), yielding language-informative visual representations. The representations are fed to the LLM as soft visual prompts to generate text conditioned on visual context. After the current task, visual features update the embedding dictionary via sparse dictionary learning. \textbf{Upper Right:} During training, the \textbf{D}ictionary \textbf{R}eplay module (\ding{184}) replays the embedding dictionary to retain the former alignment.}
\label{fig:pipeline}
\vspace{-0.2in}
\end{figure*}

\section{Problem Formulation}
This work aims to address the problem of continual alignment in pre-trained Vision-Language Models (VLMs) that incrementally align the cross-modal representations from new data distributions in two core OpenITG tasks: Image Captioning (IC) and Visual Question Answering (VQA).

Formally, the dataset for topic-$t$ (hereafter task-$t$) is $\mathcal{D}_t=\{(X_t,P_t,S_t)\}$, where $X_t = \{x_{t,i}\}_{i=1}^{n_t}$ denotes the set of input images, $P_t = \{p_{t,i}\}_{i=1}^{n_t}$ denotes the associated text inputs (prompts or questions), and $S_t = \{s_{t,i}\}_{i=1}^{n_t}$ contains ground truth sentences for each image-text pair, with $n_t$ as the number of instances in task-$t$. Given an image $x_{t, i}$ and the corresponding text $p_{t, i}$ from dataset $\mathcal{D}_t$, the VLM generates a predicted sentence, $\hat{s}_{t,i}$. Notations are detailed in Tab.~\ref{tab:notations} in Appendix~\ref{app:notation}.

Recent VLMs~\citep{alayrac2022flamingo,li2023blip,liu2023visual,bai2025qwen2} are typically composed of three components: a visual encoder, an alignment module, and a Large Language Model (LLM). Accordingly, the sentence generation process at task-$t$ can be formalized as follows,
\begin{equation}
\label{eq:general_goal}
\begin{aligned}
P(\hat{s}_{t,i}^j &\mid \hat{s}_{t,i}^{<j}, x_{t,i}, p_{t,i}) \\
&= g_{\phi_t}\!\left(\hat{s}_{t,i}^{<j},\,[A_{\omega_t}(f_{\theta_t}(x_{t,i})),\,p_{t,i}]\right),
\end{aligned}
\end{equation}
where $f_{\theta_t}(\cdot)$, $A_{\omega_t}(\cdot)$, and $g_{\phi_t}(\cdot,\cdot)$ denote the visual encoder, the alignment module, and the LLM within the VLM at task-$t$, respectively. 
To study continual alignment at the alignment module in isolation, we decompose a VLM as $(f_{\theta},A_{\omega},g_{\phi})$, and regard BLIP-2~\citep{li2023blip} as a representative model to instantiate. Following the BLIP-2, we instantiate the alignment module $A_{\omega}$ with the Q-Former to bridge a frozen visual encoder and frozen LLM.

For BLIP-2, Eq.~\ref{eq:general_goal} can be reformalized as follows,
\begin{equation}
    \label{eq:blip2}
    \begin{aligned}
    P(\hat{s}_{t,i}^j &\mid \hat{s}_{t,i}^{<j}, x_{t,i}, p_{t,i}) \\
&= g_{\phi_\star}\!\left(\hat{s}_{t,i}^{<j},\,[A_{\omega_t,Q_t}(f_{\theta_\star}(x_{t,i})),\,p_{t,i}]\right),
\end{aligned}
\end{equation}
where $\theta_\star$ and $\phi_\star$ are frozen across tasks. The alignment module $A_{\omega_t,Q_t}$ is the Q-Former at task-$t$, parameterized by $\omega_t$ and learnable query tokens $Q_t\!\in\!\mathbb{R}^{n_Q\times d_Q}$, where $n_Q$ is the number of tokens and $d_Q$ their dimension. Query outputs are linearly projected into the LLM space by a fully connected layer in $\omega_t$, enabling cross-modal alignment with frozen backbones.

Yet the learnable parameters $Q_t$ and $\omega_t$ are sensitive to distribution shifts, which causes catastrophic forgetting in IL. We therefore introduce a continual alignment mechanism to stabilize the alignment module across tasks, mitigating forgetting and supporting efficient adaptation to new data.

\vspace{-0.05in}
\section{Proposed Method}
\label{sec:model}
\subsection{Overview of Efficient Continual Alignment}\vspace{-0.05in}
To maintain accurate cross-modal alignment over sequential tasks, we propose Efficient Continual Alignment (ECA), an exemplar-free IL approach for OpenITG, and study it at the alignment module of BLIP-2. ECA comprises three components: \ding{182} \textbf{M}ixture \textbf{o}f \textbf{Q}uery (MoQ), \ding{183} \textbf{F}ish\textbf{e}r \textbf{D}ynamic \textbf{Ex}pansion (FeDEx), and \ding{184} embedding dictionary with \textbf{D}ictionary \textbf{R}eplay (DR). As shown in Fig.~\ref{fig:pipeline}, a frozen visual encoder produces patch embeddings. MoQ learns task-specific query tokens and attentively aggregates them with the pretrained query tokens. The aggregated tokens are then passed to the alignment module (Q-Former) equipped with FeDEx. FeDEx selectively expands a parallel adapter based on an FIM-based metric, so new features are incorporated while preserving established alignment. Meanwhile, DR maintains an embedding dictionary and replays it during training to retain information from previous tasks. We detail each component below.

\subsection{Mixture of Query}\label{subsec:moq}\vspace{-0.05in}
\noindent\textbf{Motivation.} 
In VLM, the alignment module can leverage learnable query tokens $Q_t$ to expose visual evidence to the frozen LLM and establish the cross-modality alignment on the new task. In incremental learning, updating $Q_t$ based solely on the current task overwrites cues learned from past tasks, which leads to catastrophic forgetting. A straightforward solution is learning task-specific query tokens separately for each task and pick the appropriate set for a given image. However, since the visual inputs in OpenITG span a wide range of scenes and contexts, visual embeddings are widely dispersed and do not cluster into discrete categories. This dispersion makes identification of a single task-specific query token set impractical, which calls for a more flexible way to reuse and combine queries across tasks.

\noindent\textbf{MoQ.} To address this challenge, we propose a novel Mixture of Query (MoQ) module, which learns a unique set of task-specific query tokens for each task and uses an attention mechanism to dynamically aggregate them. Specifically, we first learn a set of task-specific query tokens $v_t\in \mathbb{R}^{n_Q\times d_Q}$ for each task-$t$. To determine the contributions of each $v_t$, we then introduce a task-specific key $k_t\in \mathbb{R}^{d_v}$, where $d_v$ matches the dimension of $f_{\theta_{\star}}(x_{t,i})$. For each sample in task-$t$, we obtain the image embeddings as $e_{t,i}=f_{\theta_{\star}}(x_{t,i})$, and then compute the average embedding over all the patches as $\overline{e}_{t,i}$. 
After that, we use an attention mechanism $Attention(\cdot,\cdot,\cdot)$~\citep{vaswani2017attention} and combine with fixed pre-trained query tokens $Q_{\star}$ to obtain final query tokens $Q_{t,i}$ for each image $x_{t,i}$ as below,
\begin{equation}
    \label{eq:moq}
    Q_{t,i}=Q_{\star}+Attention\left(\overline{e}_{t,i}, K_t, V_t\right),
\end{equation}
where keys $K_t$ and task-specific query tokens $V_t$ are collected as $K_t = [k_1, k_2, \dots, k_t]\in\mathbb{R}^{t\times d_v}$, $V_t = [v_1, v_2, \dots, v_t]\in\mathbb{R}^{t\times n_Q\times d_Q}$ respectively. 

To preserve distinct, task-specific $v_t$ and $k_t$, the newly learned query tokens must be uncorrelated with those from previous tasks. To reduce the interference across tasks, we fix the previously learned query tokens $V_{<t}$, keys $K_{<t}$, and the pre-trained query tokens $Q_{\star}$, and enforce an orthogonality constraint between $(v_t, k_t)$ and $(V_{<t}, K_{<t})$ as follows,
\begin{equation}
    \label{eq:ortho_loss}
    \mathcal{L}_\text{orth}(k_t,v_t)=\|v_{t}V_{<t}^\top\|^2_{F}+\|k_{t}K_{<t}^\top\|^2_{F},
\end{equation}
where $\|\cdot\|_{F}$ is Frobenius norm.

In addition, to ensure that each task-specific key $k_t$ is relevant to visual embeddings in task-$t$, we optimize the key-alignment objective as follows,
\vspace{-0.05in}
\begin{equation}
    \label{eq:cos_loss}
    \mathcal{L}_\text{key}(k_t)=\frac{1}{n_t}\sum_{i=1}^{n_t}\left(1-\frac{k_t}{\|k_t\|_{2}} \frac{\overline{e}_{t,i}^\top}{\|\overline{e}_{t,i}\|_2}\right).
    \vspace{-0.05in}
\end{equation}

Finally, we optimize the MoQ as follows,
\begin{equation}
    \label{eq:moq_loss}
    \mathcal{L}_\text{MoQ}=\mathcal{L}_\text{orth}+\mathcal{L}_\text{key}.
\end{equation}
This attention-based aggregation strategy effectively integrates task-specific query tokens, ensuring robust continual alignment across diverse tasks.

By leveraging the MoQ module to adapt the pre-trained query tokens on a per-sample basis, the Q-Former alignment module is provided with an updated query token input, denoted as $Q_{t,i}$. Consequently, the alignment module can be reformulated as $A(\cdot;\omega_t, Q_{t,i})$.

\subsection{Fisher Dynamic Expansion}\label{subsec:fedexQ}\vspace{-0.05in}
\textbf{Motivation.} Fine-tuning the entire alignment module for each task is both costly and prone to degrading pre-trained performance. We therefore adopt a Parameter-Efficient Fine-Tuning (PEFT) approach with Parallel Adapters (PAs)~\citep{he2022towards}. Training a single PA is efficient, but its limited capacity may fail to capture task-specific features without harming prior alignment. One naive remedy is to assign a new PA to every task. Yet in OpenITG, images often contain a dominant semantic category of visual objects along with additional semantic categories of contextual elements. This produces overlapping feature distributions that encourage positive transfer across tasks, and blindly expanding a PA for each task can break this beneficial sharing. Consequently, it motivates a principled criterion for when expansion is truly needed.

\noindent\textbf{FIM-Based Metric.} To address this challenge, we propose a Fisher Information Matrix (FIM)-based metric to guide the dynamic expansion of Parallel Adapters (PAs) across tasks. For OpenITG tasks, the Q-Former, $A_{\omega_t, Q_{t,i}}(\cdot)$, is commonly fine-tuned by minimizing the negative log-likelihood of the probability for correctly generating the label text conditioned on task-$t$ visual input $x_{t,i}$, text input $p_{t,i}$, and the label sentence $s_{t,i}$ as follows,
\begin{equation}
\label{eq:prob}
    \begin{aligned}
    P\big(s_{t,i}^j &\mid s_{t,i}^{<j}, x_{t,i}, p_{t,i};\omega_t, Q_{t,i}\big) \\
    &\triangleq g_{\phi_{\star}}\big(s_{t,i}^{<j}, [ A(e_{t,i};\omega_t, Q_{t,i}), p_{t,i} ] \big),
    \end{aligned}
\end{equation}
\begin{equation}
\label{eq:ce_loss}
    \mathcal{L}_\text{ce}\left(\omega_t,Q_{t,i}\right) = -\sum_{j=1}^{|s_{t,i}|}\log P\left(s_{t,i}^j|s_{t,i}^{<j},x_{t,i},p_{t,i};\omega_t, Q_{t,i}\right),
\end{equation}
which provides a basis for subsequent analysis.

Our goal is to measure how $\omega_t$ updated by Eq.~\ref{eq:ce_loss} on dataset $\mathcal{D}_{t+1}$ for task-$t+1$ affects performance on $\mathcal{D}_t$ for task-$t$. To this end, we first obtain the $\mathcal{L}_{\mathcal{D}_{t}}$ as the expectation of $\mathcal{L}_\text{ce}$ on dataset $\mathcal{D}_{t}$ by $\mathcal{L}_{\mathcal{D}_t}=\mathbb{E}_{\mathcal{D}_t}[\mathcal{L}_\text{ce}]$. Then, we derive a FIM-based metric $S(\omega_t)$ based on the second-order Taylor expansion with FIM approximation (see Appendix~\ref{app:derivation}), which reflects the conflict between the parameter updates for the new task and the preservation of prior knowledge. 

\begin{definition}\label{def:impact}
  We define the incremental and decremental impacts of a small gradient update at $\omega_t$ on dataset $\mathcal{D}_t$ as $I_+(\omega_t)\geq 0$ and $I_-(\omega_t)\leq 0$, respectively.
  \vspace{-0.1in}
\end{definition}

\begin{restatable}{theorem}{conflict}
\label{theo:conflict}
Based on Def.~\ref{def:impact}, the performance degradation on dataset $\mathcal{D}_t$ is denoted as 
$\Delta\mathcal{L}_{\mathcal{D}_t}(\Delta\omega)=I_+(\omega_t)+I_-(\omega_t)$.
Define the normalized FIM-based metric as
\begin{equation}
\label{eq:metric}
S(\omega_t) = \frac{I_{+}(\omega_t)}{I_{+}(\omega_t)+|I_{-}(\omega_t)|} \in [0,1]
\end{equation}
Then, under a small-step update and the Fisher approximation, we have
\begin{itemize}
    \item If \(S(\omega_t) \le 0.5\), training on \(\mathcal{D}_{t+1}\) does not degrade the performance on \(\mathcal{D}_t\) (i.e., \(\Delta \mathcal{L}_{\mathcal{D}_t} \le 0\)).
    \item If \(S(\omega_t) > 0.5\), the update on \(\mathcal{D}_{t+1}\) will degrade the performance on \(\mathcal{D}_t\).
\end{itemize}
\end{restatable}

We provide derivation and proof in Appendix~\ref{app:derivation} that theoretically $S(\omega_t)\le 0.5$ guarantees non-degradation on the previous task. Moreover, Appendix~\ref{app:conflict_exp} reports $S(\omega_t)$ sweeps showing that $0.5$ yields the best performance in our settings.

\noindent\textbf{Dynamic adapter expansion.} Finally, we employ $S(\omega_t)$ to decide whether naive reuse of the current PA on $\mathcal{D}_{t+1}$ would severely impact $\mathcal{D}_t$. When the impact metric surpasses $0.5$, we expand the Q-Former with a new PA. After expanding the PA for a new task, we freeze the previously trained PAs and compute the mean of all PA outputs as the final output. 
FeDEx enables the Q-Former to retain the alignment of previous tasks while maintaining parameter efficiency in the training process.

\subsection{Dictionary Replay}\label{subsec:dictionary}
\textbf{Motivation}. Although the MoQ module and FeDEx integrate new task-specific knowledge while preserving previously learned knowledge, the absence of data from former tasks remains a critical challenge for exemplar-free IL. To address that challenge in conventional classification settings, various methods~\citep{zhu2022self, petit2023fetril, zhu2021prototype} have shown that a single prototype per category can effectively serve as a compact memory representation for data from the previous tasks. However, as discussed in Sec.~\ref {subsec:moq}, the diverse and dispersed nature of visual embeddings in OpenITG makes a single prototype insufficient.

\noindent\textbf{Embedding Dictionary Update.} To overcome this limitation, we propose an embedding dictionary, a memory mechanism that decouples visual embeddings and captures their essential components through dictionary learning. Specifically, we learn an over-complete dictionary $D\in \mathbb{R}^{m\times d_v}$ ($m\gg d_v$) so each patch embedding $e_k \in \mathbb{R}^{d_v}$ from any image can be reconstructed as a sparse linear combination of a few rows from $D$ via solving the following Lasso problem,
\begin{equation}
    \label{eq:lasso}    \alpha_{k}=\mathop{\arg\min}_{\alpha\in\mathbb{R}^{m}}\frac{1}{2}\|e_k-D_{t-1}^\top\alpha\|_{F}^2+\gamma\|\alpha\|_1, \textbf{ s.t. } \alpha\geq 0,
\end{equation}
where $\gamma$ controls the sparsity of the reconstruction coefficients $\alpha_k$.
The Lasso problem can be efficiently solved by FISTA~\citep{beck2009fast}. 

After obtaining $\alpha_{k}$, we minimize the reconstruction error to update the embedding dictionary,
\begin{equation}
    \label{eq:update_d}
    D_{t}=\mathop{\arg\min}_{D\in\mathbb{R}^{m\times d_v}}\frac{1}{2}\|e_k-D^\top\alpha_{k}\|_{F}^2, \textbf{ s.t. } \|a_{j}\|_2\leq1, \forall j\in [m],
\end{equation}
where $a_j\in\mathbb{R}^{1\times d_v}$ denotes $j$-th atom in $D$. The unit-norm constraint $\|a_j\|_2\le 1$ prevents any atom from growing too large, and removes scale ambiguity between $D$ and $\alpha_k$. With this constraint, we fix $\gamma=1$ to obtain consistently sparse codes across tasks.

\noindent\textbf{Embedding Dictionary Replay.} With the updated embedding dictionary $D_{t}$ capturing the key components of visual embeddings across $t$ tasks, we replay this dictionary to retain previous task knowledge in future task training. Let $sg(\Omega_{t})=sg(\{\omega_{t}, K_{t}, V_{t}\})$ represent the fixed updated parameters, where $sg(\cdot)$ means the stop-gradient operator. $\Omega_{t}$ includes updated Q-Former parameters, keys, and query tokens obtained after training task-$t$. To optimize the parameters for the future task-$t+1$, represented by $\Omega_{t+1}$, we apply the knowledge distillation loss as follows,
\begin{equation}
    \label{eq:kd}
    \mathcal{L}_\text{DR}(\Omega_{t+1})=\frac{1}{m} \left\| A(D_{t}; sg(\Omega_{t})) - A(D_{t}; \Omega_{t+1}) \right\|_F^2.
\end{equation}
Replaying the former data encoded in the embedding dictionary ensures that the former alignment is preserved in the future task.

\subsection{Objective Function}
Finally, to optimize the model in the task-$t$, we first use $S(\omega_t)$ in Eq.~\ref{eq:metric} to expand the Q-Former equipped with FeDEx, and then we jointly train the model parameters $\omega_t$, keys $k_t$, and values $v_t$ by the loss function with each sample in $\mathcal{D}_t$ as follows,
\begin{equation}
    \label{eq:final_loss}
    \mathcal{L}=\mathcal{L}_\text{ce}+\mathcal{L}_\text{MoQ}+\lambda\mathcal{L}_\text{DR},
\end{equation}
where $\lambda$ is a hyper-parameter that balances the contributions of the dictionary replay loss $\mathcal{L}_\text{DR}$. After training task-$t$, we apply the dictionary learning to update the embedding dictionary for embedding the essential elements of task-$t$'s visual embeddings.



\section{Experiments}\label{sec:experiment}
To assess the proposed ECA, we first construct four new IL benchmarks for OpenITG tasks, ToS-COCO Caption, ToS-VQAv2, ToS-TextCaps, ToS-TextVQA, based on two well-known image captioning datasets, \ie COCO ImageCaption~\citep{lin2014microsoft}, TextCaps~\citep{sidorov2019textcaps}, and two well-known VQA datasets, \ie VQAv2~\citep{goyal2017making}, TextVQA~\citep{singh2019towards}. Next, we compare our proposed ECA with state-of-the-art (SOTA) exemplar-free methods on four new IL benchmarks. Lastly, we perform ablation studies to explore the impact of key components.

\noindent
\textbf{Benchmarks.}
To evaluate the performance of our method, we construct four IL benchmarks for OpenITG by splitting tasks based on ``main topic'': ToS-COCO Caption and ToS-VQAv2 from COCO Caption/VQAv2, and ToS-TextCaps and ToS-TextVQA from TextCaps/TextVQA. These splits preserve realistic overlap, where earlier main topics reappear as context in later tasks Fig.~\ref{fig:dataset_dis}. Benchmark construction details are presented in Appendix~\ref{app:benchmark}.

\begin{table*}[!t]
    \centering
    \vspace{-0.05in}
    \caption{Evaluation on ToS-COCO Caption and ToS-VQAv2. \textbf{Bold}: Best results on each dataset. \underline{Underline}: Second best results on each dataset. ``Avg'': Final Average performance; ``BWT'': Backward Transfer; ``FWT'': Forward Transfer. }
    \vspace{-0.05in}
    \label{tab:ms-sota}
    \begin{adjustbox}{width=1\textwidth}
    \setlength{\tabcolsep}{1.2pt}
    \begin{tabular}{l c ccc ccc ccc c ccc}
        \toprule
        \multicolumn{1}{l}{\textbf{Tasks}} & \multicolumn{10}{c}{\textbf{ToS-COCO Caption}} & \multicolumn{4}{c}{\textbf{ToS-VQAv2}} \\ \cmidrule(lr){1-1} \cmidrule(lr){2-11} \cmidrule(lr){12-15}
        \multirow{2}{*}{\textbf{Method}} & 
        \multirow{2}{*}{\textbf{\makecell[c]{\# Trainable\\ Params}}} & 
        \multicolumn{3}{c}{\textbf{BLEU-4}} & 
        \multicolumn{3}{c}{\textbf{CIDEr}} & 
        \multicolumn{3}{c}{\textbf{SPICE}} & 
        \multirow{2}{*}{\textbf{\makecell[c]{\# Trainable\\ Params}}} & 
        \multicolumn{3}{c}{\textbf{VQA Acc}} \\
        \cmidrule(lr){3-5} \cmidrule(lr){6-8} \cmidrule(lr){9-11} \cmidrule(lr){13-15}
         & & Avg $\uparrow$ & BWT $\uparrow$ & FWT $\uparrow$  
         & Avg $\uparrow$ & BWT $\uparrow$ & FWT $\uparrow$  
         & Avg $\uparrow$ & BWT $\uparrow$ & FWT $\uparrow$
         & & Avg $\uparrow$ & BWT $\uparrow$ & FWT $\uparrow$ \\ \midrule
        ZeroShot & 0 M & 36.00 & - & - & 104.65 & - & - & 21.12 & - & - & 0 M & 48.33 & - & - \\
        Vanilla (PA) & 12.29 M & 42.70 & -1.49 & 6.48 & 123.00 & -4.50 & 19.15 & 23.39 & -0.78 & 2.64 & 21.74 M & 64.39 & -2.00 & 12.02 \\
        Vanilla (Q-Former) & 107.13 M & 42.21 & -2.02 & 6.60 & 123.66 & -4.58 & 19.23 & 23.42 & -0.85 & 2.68 & 163.82 M & 64.14 & -1.92 & 11.78 \\
        LwF~\citep{li2017learning} & 12.29 M & 42.91 & -1.07 & 6.56 & 123.88 & -3.78 & 19.20 & 23.51 & -0.69 & \underline{2.71} & 21.74 M & 64.92 & -0.99 & \underline{14.65} \\
        EWC~\citep{lee2017overcoming} & 12.29 M & 42.86 & -1.45 & 6.62 & 123.66 & -4.03 & 18.73 & \underline{23.55} & -0.57 & 2.55 & 21.74 M & 59.63 & -2.32 & 11.33 \\
        Dual-Prompt~\citep{wang2022dualprompt} & 14.30 M & 43.03 & \textbf{-0.62} & 6.77 & 123.59 & \underline{-1.60} & 19.04 & 23.47 & -0.52 & 2.65 & 21.67 M & 65.03 & 1.27 & 12.74 \\
        CODA-Prompt~\citep{smith2023coda} & 15.41 M & \underline{43.10} & -0.67 & \underline{6.90} & \underline{124.20} & \textbf{-1.19} & \underline{19.44} & 23.52 & \underline{-0.38} & 2.59 & 24.37 M & \underline{65.64} & \underline{1.38} & 13.71 \\
        MoE-LoRA~\citep{liu2024moe} & 98.84 M & 42.20 & -1.56 & 6.25 & 122.77 & -3.53 & 17.76 & 23.39 & -0.68 & 2.51 & 195.71 M & 61.02 & -3.90 & 10.27 \\
        \midrule
        \textbf{ECA (Ours)} & 12.29 M & 
                              \textbf{43.42} & \underline{-0.64} & \textbf{7.39} & 
                              \textbf{125.56} & -1.86 & \textbf{20.58} & 
                              \textbf{23.80} & \textbf{-0.35} & \textbf{3.00} & 
                              21.74 M &
                              \textbf{68.05} & \textbf{1.81} & \textbf{16.38} \\
        \midrule
        Upper-bound (PA) & 12.29 M & 43.94 & - & - & 126.91 & - & - & 24.18 & - & - & 21.74 M & 68.18 & - & - \\
        Upper-bound (Q-Former) & 107.13 M & 43.82 & - & - & 126.19 & - & - & 24.10 & - & - & 163.82 M & 68.52 & - & - \\
        \bottomrule
    \end{tabular}
    \end{adjustbox}
    \vspace{-0.15in}
\end{table*}

\noindent
\textbf{Protocol.} Following the common protocol of the OpenITG tasks~\citep{li2023blip,del2020ratt,antol2015vqa}, we use metrics BLEU@4, CIDEr, and SPICE for Image Captioning tasks, and VQA Accuracy for open-ended VQA. 
We assess IL performance using three metrics: \emph{Average Performance} (Avg), \emph{Forward Transfer} (FWT), and \emph{Backward Transfer} (BWT), following the protocol in~\citep{lopez2017gradient}. We report these metrics to summarize final performance across all tasks, quantify how learning later tasks affects earlier ones, and measure transfer to unseen tasks, respectively. Definitions are in Appendix~\ref{app:ep_protocol}. 

\noindent
\textbf{Baselines.} 
We apply a pre-trained BLIP-2~\citep{li2023blip} as a backbone, and instantiate \emph{all methods} on the Q-Former from the BLIP-2 with the visual encoder and the LLM frozen for fairness. Under IL settings, we evaluate ``ZeroShot,'' ``Vanilla (PA),'' finetuning one Parallel Adapter (PA) on Q-Former sequentially, ``Upper-bound (PA),'' jointly finetuning one PA across tasks, ``Vanilla (Q-Former),'' finetuning the Q-Former sequentially, and ``Upper-bound (Q-Former),'' jointly finetuning the Q-Former across tasks. To show our performance, we also compare several state-of-the-art exemplar-free IL methods, originally developed for uni-modal tasks, namely EWC~\citep{lee2017overcoming}, LwF~\citep{li2017learning}, CODA-Prompt~\citep{smith2023coda}, Dual-Prompt~\citep{wang2022dualprompt}, and one multi-modal IL method, MoE-LoRA~\citep{liu2024moe} (following~\citep{chen2025coin}). LwF~\citep{li2017learning} and EWC~\citep{lee2017overcoming} are applied to the ``Vanilla (PA)'' under the same trainable scope as ECA for fair comparison, while other methods follow their original configuration. 
Full training details are provided in Appendix~\ref{app:ep_detail}.
\subsection{Main Results}\label{subsec:main_result}
In this section, we compare the overall performance of our proposed method, ECA, with various baselines on our proposed benchmarks under the IL setting.

\begin{table*}[!t]
    \centering
    \vspace{-0.05in}
    \caption{Evaluation on ToS-TextCaps and ToS-TextVQA. \textbf{Bold}: Best results on each dataset. \underline{Underline}: Second best results on each dataset. ``Avg'': Final Average performance; ``BWT'': Backward Transfer; ``FWT'': Forward Transfer.}
    \vspace{-0.05in}
    \label{tab:text-sota}
    \begin{adjustbox}{width=1\textwidth}
    \setlength{\tabcolsep}{1.2pt}
    \begin{tabular}{l c ccc ccc ccc ccc}
        \toprule
        \multicolumn{2}{l}{\textbf{Tasks}} & \multicolumn{9}{c}{\textbf{ToS-TextCaps}} & \multicolumn{3}{c}{\textbf{ToS-TextVQA}} \\ \cmidrule(lr){1-2} \cmidrule(lr){3-11} \cmidrule(lr){12-14}
        \multirow{2}{*}{\textbf{Method}} & 
        \multirow{2}{*}{\textbf{\makecell[c]{\# Trainable\\ Params}}} & 
        \multicolumn{3}{c}{\textbf{BLEU-4}} & 
        \multicolumn{3}{c}{\textbf{CIDEr}} & 
        \multicolumn{3}{c}{\textbf{SPICE}} & 
        \multicolumn{3}{c}{\textbf{VQA Acc}} \\
        \cmidrule(lr){3-5} \cmidrule(lr){6-8} \cmidrule(lr){9-11} \cmidrule(lr){12-14}
        & & Avg $\uparrow$ & BWT $\uparrow$ & FWT $\uparrow$ 
        & Avg $\uparrow$ & BWT $\uparrow$ & FWT $\uparrow$ 
        & Avg $\uparrow$ & BWT $\uparrow$ & FWT $\uparrow$ 
        & Avg $\uparrow$ & BWT $\uparrow$ & FWT $\uparrow$ \\
        \midrule
        ZeroShot & 0 M & 13.99 & - & - & 48.65 & - & - & 11.48 & - & - & 14.83 & - & - \\
        Vanilla (PA) & 21.74 M & 24.50 & -3.14 & 9.89 & 89.39 & -2.76 & 33.61 & 15.19 & -1.47 & 3.32 & 24.94 & -3.66 & 10.85 \\
        Vanilla (Q-Former) & 163.82 M & 26.98 & -1.56 & 8.66 & 93.63 & -1.84 & 28.90 & 15.90 & -0.87 & 3.29 & 32.13 & -1.37 & 15.03 \\
        LwF~\citep{li2017learning} & 21.74 M & \underline{27.89} & -0.59 & \underline{10.89} & \underline{97.46} & \underline{1.90} & \underline{36.46} & \underline{16.19} & -0.05 & \underline{3.42} & \underline{32.92} & -0.53 & \underline{15.52} \\
        EWC~\citep{lee2017overcoming} & 21.74 M & 23.90 & -2.67 & 9.16 & 86.98 & -5.87 & 32.34 & 14.60 & -1.13 & 2.64 & 30.21 & 0.54 & 11.67 \\
        Dual-Prompt~\citep{wang2022dualprompt} & 19.82 M & 23.67 & \underline{-0.34} & 6.34 & 83.47 & 1.15 & 22.12 & 14.69 & \textbf{0.79} & 2.01 & 25.64 & \underline{1.75} & 8.14 \\ 
        CODA-Prompt~\citep{smith2023coda} & 22.13 M & 24.81 & -0.48 & 7.27 & 86.33 & 1.08 & 24.70 & 15.29 & \underline{0.77} & 2.31 & 26.13 & 1.35 & 8.83 \\
        MoE-LoRA~\citep{liu2024moe} & 195.71 M & 25.16 & -0.94 & 10.10 & 90.58 & -1.54 & 34.39 & 15.80 & 0.11 & 3.24 & 31.76 & -3.20 & 12.40 \\
        \midrule
        \textbf{ECA (Ours)} & 21.74 M & \textbf{30.05} & \textbf{-0.18} & \textbf{12.13} & \textbf{103.03} & \textbf{1.94} & \textbf{39.22} & \textbf{16.86} & 0.14 & \textbf{4.39} & \textbf{38.13} & \textbf{2.36} & \textbf{19.30} \\
        \midrule
        Upper-bound (PA) & 21.74 M & 30.59 & - & - & 110.49 & - & - & 17.78 & - & - & 41.05 & - & - \\
        Upper-bound (Q-Former) & 163.82 M & 31.32 & - & - & 111.99 & - & - & 18.02 & - & - & 46.02 & - & - \\
        \bottomrule
    \end{tabular}
    \end{adjustbox}
    \vspace{-0.15in}
\end{table*}
\noindent
\textbf{Evaluation on ToS-COCO Caption and ToS-VQAv2.}
As shown in Tab.~\ref{tab:ms-sota}, ECA significantly outperforms other baselines in terms of Avg, BWT, and FWT on both caption metrics and VQA accuracy, while using fewer trainable parameters.
On ToS-COCO Caption, since BLIP-2 is pre-trained on COCO Caption, the absolute gain is modest, yet the Upper-Bound Gap Closed (UBGC)\footnote{\scriptsize\textbf{UBGC}
$=(\text{ECA}-\text{method})/(\text{Upper-bound (PA)}-\text{method})\times100\%$;
UB is the oracle joint-training upper bound on the union of tasks.} is remarkable. ECA’s UBGC relative to CODA-Prompt, the SOTA uni-modal exemplar-free IL method, in Avg are $38.10\%$, $50.18\%$, and $42.42\%$ for BLEU-4, CIDEr, and SPICE, respectively. Although ECA has slightly lower BWT on CIDEr than CODA-Prompt and Dual-Prompt, its higher Avg and closeness to the upper bound better reflect final IL quality. For ToS-VQAv2, whose annotations are not part of BLIP-2’s pre-training, ECA is only $0.13$ below the \textbf{upper-bound} performance and surpasses CODA-Prompt by $2.41$ in Avg for VQA accuracy.

\noindent
\textbf{Evaluation on ToS-TextCaps and ToS-TextVQA}
These two benchmarks form a harder continual-alignment setting because BLIP-2 was not pre-trained on them, and successful predictions rely on OCR tokens that must interact with visual features through cross-attention. As shown in Tab.~\ref{tab:text-sota}, ECA remains strong on both datasets. On ToS-TextCaps, compared with LwF as the best uni-modal exemplar-free baseline and with the ``Upper-bound (PA)'' which uses the same trainable scope under joint training, ECA improves Avg by $2.16$, $5.57$, and $0.67$ for BLEU-4, CIDEr, and SPICE, with UBGC of $80.00\%$, $42.74\%$, and $42.14\%$. On ToS-TextVQA, ECA improves Avg VQA accuracy over LwF by $5.21$, with a UBGC of $64.08\%$.


\noindent
\textbf{Findings}
We further examine these exemplar-free methods and observe several notable points:
\textbf{(1).} As shown in Tab.~\ref{tab:ms-sota} and Tab.~\ref{tab:text-sota}, LwF surpasses prompt based baselines on ToS-TextCaps and ToS-TextVQA. These datasets require reasoning over OCR tokens, which induces a larger distribution shift beyond the pretraining regime. Prompt pools do not explicitly preserve the alignment for the newly introduced tokens, whereas LwF maintains the cross-modal alignment through knowledge distillation.
\textbf{(2).} EWC is misaligned with our main topic setting. Classical EWC presumes disjoint tasks and thus restricts updates to parameters deemed important for past tasks. In our scenario, cross-task semantic overlap means such updates can be beneficial. Enforcing these constraints suppresses useful sharing and harms transfer. As shown in Tab.~\ref{tab:text-sota}, EWC even underperforms the `Vanilla (PA).''
\textbf{(3).} As shown in Tab.~\ref{tab:ms-sota} and Tab.~\ref{tab:text-sota}, ECA obtains impressive BWT and FWT across all metrics on four datasets. This indicates that ECA not only mitigates catastrophic forgetting but also uses prior knowledge to better handle future tasks. Namely, ECA attains an understanding of tasks with different main topics and generalizes well to upcoming tasks.
\textbf{(4).} As shown in Tab.~\ref{tab:ms-sota} and Tab.~\ref{tab:text-sota}, ECA uses almost the same number of trainable parameters as the baselines with single PA, yet consistently outperforms them and even methods with much larger parameter budgets. This suggests that ECA better exploits a limited alignment capacity rather than blindly increasing model size and computational cost. Furthermore, we show its parameter and inference efficiency in Appendix~\ref{app:param-eff}. In sum, ECA delivers superior overall performance and achieves continual alignment for IL in OpenITG with efficient parameter usage.

In addition, to validate generality beyond Q-Former-based VLMs, we evaluate ECA on the projector-based MLLM LLaVA-v0 and compare with PA baselines, MoE-LoRA, and the recent ModalPrompt~\cite{zeng2025modalprompt} in Appendix~\ref{app:eca_llava}. Additional case studies are in Appendix~\ref{app:case_study}.

\vspace{-0.05in}
\subsection{Ablation Studies}
\textbf{Effect of key components \& Hyper-parameters.}
We first study the effect of key components of ECA in Tab.~\ref{tab:ablation}. MoQ improves Avg and BWT over ``Vanilla (PA)'' by sharing orthogonal query tokens across tasks, DR further boosts, especially FWT via dictionary replay, and FeDEx mitigates forgetting by expanding adapters only when needed. The analyses are in Appendix~\ref{app:add-ablation}. Additional ablation for losses in MoQ is in Tab.~\ref{tab:ablation_moqloss}, Appendix~\ref{app:add-ablation}. Moreover, we study the influence of hyper-parameters, \ie DR weight $\lambda$, DR's embedding dictionary atom number $m$ in Appendix~\ref{app:hyper-parameter}.

\begin{table}[!t]
    \centering
    \caption{Ablation study on ToS-COCO Caption.\textbf{``Naive-Q:''} per-task query tokens without cross-task sharing (one set per task). \textbf{``DR(r):''} replay a randomly initialized dictionary (no dictionary).}
    \label{tab:ablation}
    \setlength{\tabcolsep}{1.5pt}
    \begin{adjustbox}{width=\linewidth}
    \begin{tabular}{l ccc ccc ccc}
        \toprule
        \multirow{2}{*}{\textbf{Method}} & 
        \multicolumn{3}{c}{\textbf{BLEU-4}} & 
        \multicolumn{3}{c}{\textbf{CIDEr}} & 
        \multicolumn{3}{c}{\textbf{SPICE}} \\
        \cmidrule(lr){2-4} \cmidrule(lr){5-7} \cmidrule(lr){8-10}
         & Avg $\uparrow$ & BWT $\uparrow$ & FWT $\uparrow$  
         & Avg $\uparrow$ & BWT $\uparrow$ & FWT $\uparrow$  
         & Avg $\uparrow$ & BWT $\uparrow$ & FWT $\uparrow$ \\
        \midrule
        Vanilla (PA)           & 42.70 & -1.49 & 6.48 & 123.00 & -4.50 & 19.15 & 23.39 & -0.78 & 2.64 \\
        PA+Naive-Q         & 42.37 & -1.88 & 6.91 &122.74 & -4.27 & 19.29 & 23.33 & -0.79 & 2.82 \\
        PA+MoQ              & 42.80 & -1.25 & 6.77 & 123.67 & -3.66 & 19.05 & 23.47 & -0.59 & 2.71 \\
        PA+MoQ+DR           & 42.97 & -1.16 & 7.28 & 124.57 & -2.80 & 20.45 & 23.59 & -0.54 & 2.96 \\
        PA+MoQ+DR(r)   & 42.49 & -1.57 & 7.24 & 123.24 & -3.75 & 19.88 & 23.57 & -0.66 & 2.86 \\
        PA+MoQ+FeDEx       & 43.22	& -0.72 & 7.05 & 124.95 & -2.04 & 19.72 & 23.69 & -0.42 & 2.83 \\
        \midrule
        \textbf{ECA (Ours)} & \textbf{43.42} & \textbf{-0.64} & \textbf{7.39} & 
                              \textbf{125.56} & \textbf{-1.86} & \textbf{20.58} & 
                              \textbf{23.80} & \textbf{-0.35} & \textbf{3.00} \\
        \bottomrule
    \end{tabular}
    \end{adjustbox}
    \vspace{-0.1in}
\end{table}
\section{Conclusion and Future Work}\label{sec:conclusion}

In this work, we introduced the notion of continual alignment for incremental learning in open-ended image-to-text generation and proposed Efficient Continual Alignment (ECA), an exemplar-free framework for adapting the alignment module in VLMs. ECA employs three key components, the Mixture of Query (MoQ) module, Fisher Dynamic Expansion (FeDEx), and Dictionary Replay (DR), which enable a pre-trained VLM to acquire new task-specific features while preserving robust cross-modal alignment without storing exemplars. We also constructed four new IL benchmarks that reflect realistic distribution shifts in OpenITG, and experiments show that ECA can significantly mitigate catastrophic forgetting and achieve strong performance with high parameter efficiency. While these results are encouraging, DR maintains a fixed-size embedding dictionary across tasks in an online setting. This dictionary may be insufficient for very long task sequences with highly diverse visual distributions, and frequently reused atoms may be updated by later tasks and overwrite earlier representations. Our current instantiation also assumes a reasonably strong pre-trained VLM backbone that provides high-quality representations. Future work includes extending DR to a dynamic and adaptive dictionary, and extending ECA to weaker backbones or coupling pre-training with continual alignment.



\section*{Acknowledgments}
Research reported in this paper was sponsored in part by NSF CPS 2311086, NSF CIRC 716152, NSF RITEL 2506890, NAIRR 250288, and Faculty Research Grant at William \& Mary 141446. Part of this work was conducted while Jiangtao Kong was an intern at \mbox{JPMorganChase}. We thank the Global Technology Applied Research center of \mbox{JPMorganChase} for helpful support.

\section*{Impact Statement}
This paper studies exemplar-free incremental learning for open-ended image-to-text generation. We update only the alignment module of a pre-trained vision-language model such as BLIP-2. We evaluate methods on our proposed  IL OpenITG benchmarks, which are built on public datasets, in Sec.~\ref{sec:experiment}. These datasets include MSCOCO Caption, VQAv2, TextCaps, and TextVQA. We do not collect new human data. We do not use private datasets. We do not design or evaluate methods for identifying individuals. We expect the main impact to be scientific. Our setting, benchmarks, and method can support more rigorous evaluation of continual multimodal generation. As with vision-language generation methods in general, downstream use should still consider robustness and bias.

\section*{Disclaimer}
This paper was prepared for informational purposes by the Global Technology Applied Research center of JPMorganChase. This paper is not a product of the Research Department of JPMorganChase or its affiliates. Neither JPMorganChase nor any of its affiliates makes any explicit or implied representation or warranty, and none of them accepts any liability in connection with this paper, including, without limitation, with respect to the completeness, accuracy, or reliability of the information contained herein and the potential legal, compliance, tax, or accounting effects thereof. This document is not intended as investment research or investment advice, or as a recommendation, offer, or solicitation for the purchase or sale of any security, financial instrument, financial product, or service, or to be used in any way for evaluating the merits of participating in any transaction.


\bibliography{icml2026}

@inproceedings{radford2021learning,
  title={Learning transferable visual models from natural language supervision},
  author={Radford, Alec and Kim, Jong Wook and Hallacy, Chris and Ramesh, Aditya and Goh, Gabriel and Agarwal, Sandhini and Sastry, Girish and Askell, Amanda and Mishkin, Pamela and Clark, Jack and others},
  booktitle={International conference on machine learning},
  pages={8748--8763},
  year={2021},
  organization={PMLR}
}

@article{alayrac2022flamingo,
  title={Flamingo: a visual language model for few-shot learning},
  author={Alayrac, Jean-Baptiste and Donahue, Jeff and Luc, Pauline and Miech, Antoine and Barr, Iain and Hasson, Yana and Lenc, Karel and Mensch, Arthur and Millican, Katherine and Reynolds, Malcolm and others},
  journal={Advances in neural information processing systems},
  volume={35},
  pages={23716--23736},
  year={2022}
}

@inproceedings{li2022blip,
  title={Blip: Bootstrapping language-image pre-training for unified vision-language understanding and generation},
  author={Li, Junnan and Li, Dongxu and Xiong, Caiming and Hoi, Steven},
  booktitle={International conference on machine learning},
  pages={12888--12900},
  year={2022},
  organization={PMLR}
}

@inproceedings{li2023blip,
  title={Blip-2: Bootstrapping language-image pre-training with frozen image encoders and large language models},
  author={Li, Junnan and Li, Dongxu and Savarese, Silvio and Hoi, Steven},
  booktitle={International conference on machine learning},
  pages={19730--19742},
  year={2023},
  organization={PMLR}
}

@inproceedings{antol2015vqa,
  title={Vqa: Visual question answering},
  author={Antol, Stanislaw and Agrawal, Aishwarya and Lu, Jiasen and Mitchell, Margaret and Batra, Dhruv and Zitnick, C Lawrence and Parikh, Devi},
  booktitle={Proceedings of the IEEE international conference on computer vision},
  pages={2425--2433},
  year={2015}
}

@inproceedings{vinyals2015show,
  title={Show and tell: A neural image caption generator},
  author={Vinyals, Oriol and Toshev, Alexander and Bengio, Samy and Erhan, Dumitru},
  booktitle={Proceedings of the IEEE conference on computer vision and pattern recognition},
  pages={3156--3164},
  year={2015}
}

@inproceedings{he2022towards,
    title={Towards a Unified View of Parameter-Efficient Transfer Learning},
    author={Junxian He and Chunting Zhou and Xuezhe Ma and Taylor Berg-Kirkpatrick and Graham Neubig},
    booktitle={International Conference on Learning Representations},
    year={2022},
    url={https://openreview.net/forum?id=0RDcd5Axok}
}

@inproceedings{zhu2022self,
  title={Self-sustaining representation expansion for non-exemplar class-incremental learning},
  author={Zhu, Kai and Zhai, Wei and Cao, Yang and Luo, Jiebo and Zha, Zheng-Jun},
  booktitle={Proceedings of the IEEE/CVF Conference on Computer Vision and Pattern Recognition},
  pages={9296--9305},
  year={2022}
}

@inproceedings{petit2023fetril,
  title={Fetril: Feature translation for exemplar-free class-incremental learning},
  author={Petit, Gr{\'e}goire and Popescu, Adrian and Schindler, Hugo and Picard, David and Delezoide, Bertrand},
  booktitle={Proceedings of the IEEE/CVF winter conference on applications of computer vision},
  pages={3911--3920},
  year={2023}
}

@inproceedings{zhu2021prototype,
  title={Prototype augmentation and self-supervision for incremental learning},
  author={Zhu, Fei and Zhang, Xu-Yao and Wang, Chuang and Yin, Fei and Liu, Cheng-Lin},
  booktitle={Proceedings of the IEEE/CVF Conference on Computer Vision and Pattern Recognition},
  pages={5871--5880},
  year={2021}
}

@article{beck2009fast,
  title={A fast iterative shrinkage-thresholding algorithm for linear inverse problems},
  author={Beck, Amir and Teboulle, Marc},
  journal={SIAM journal on imaging sciences},
  volume={2},
  number={1},
  pages={183--202},
  year={2009},
  publisher={SIAM}
}

@article{del2020ratt,
  title={Ratt: Recurrent attention to transient tasks for continual image captioning},
  author={Del Chiaro, Riccardo and Twardowski, Bart{\l}omiej and Bagdanov, Andrew and Van de Weijer, Joost},
  journal={Advances in Neural Information Processing Systems},
  volume={33},
  pages={16736--16748},
  year={2020}
}

@inproceedings{zhang2023vqacl,
  title={Vqacl: A novel visual question answering continual learning setting},
  author={Zhang, Xi and Zhang, Feifei and Xu, Changsheng},
  booktitle={Proceedings of the IEEE/CVF Conference on Computer Vision and Pattern Recognition},
  pages={19102--19112},
  year={2023}
}

@inproceedings{lei2023symbolic,
  title={Symbolic replay: Scene graph as prompt for continual learning on vqa task},
  author={Lei, Stan Weixian and Gao, Difei and Wu, Jay Zhangjie and Wang, Yuxuan and Liu, Wei and Zhang, Mengmi and Shou, Mike Zheng},
  booktitle={Proceedings of the AAAI Conference on Artificial Intelligence},
  volume={37},
  number={1},
  pages={1250--1259},
  year={2023}
}

@InProceedings{Ramos_2023_CVPR,
    author    = {Ramos, Rita and Martins, Bruno and Elliott, Desmond and Kementchedjhieva, Yova},
    title     = {SmallCap: Lightweight Image Captioning Prompted With Retrieval Augmentation},
    booktitle = {Proceedings of the IEEE/CVF Conference on Computer Vision and Pattern Recognition (CVPR)},
    month     = {June},
    year      = {2023},
    pages     = {2840-2849}
}

@article{herdade2019image,
  title={Image captioning: Transforming objects into words},
  author={Herdade, Simao and Kappeler, Armin and Boakye, Kofi and Soares, Joao},
  journal={Advances in neural information processing systems},
  volume={32},
  year={2019}
}

@inproceedings{Xu2015ShowAA,
  title={Show, Attend and Tell: Neural Image Caption Generation with Visual Attention},
  author={Ke Xu and Jimmy Ba and Ryan Kiros and Kyunghyun Cho and Aaron C. Courville and Ruslan Salakhutdinov and Richard S. Zemel and Yoshua Bengio},
  booktitle={International Conference on Machine Learning},
  year={2015},
  url={https://api.semanticscholar.org/CorpusID:1055111}
}

@inproceedings{xu-etal-2020-open,
    title = "Open-Ended Visual Question Answering by Multi-Modal Domain Adaptation",
    author = "Xu, Yiming  and
      Chen, Lin  and
      Cheng, Zhongwei  and
      Duan, Lixin  and
      Luo, Jiebo",
    editor = "Cohn, Trevor  and
      He, Yulan  and
      Liu, Yang",
    booktitle = "Findings of the Association for Computational Linguistics: EMNLP 2020",
    month = nov,
    year = "2020",
    address = "Online",
    publisher = "Association for Computational Linguistics",
    url = "https://aclanthology.org/2020.findings-emnlp.34",
    doi = "10.18653/v1/2020.findings-emnlp.34",
    pages = "367--376",
}

@inproceedings{fu-etal-2023-generate,
    title = "Generate then Select: Open-ended Visual Question Answering Guided by World Knowledge",
    author = "Fu, Xingyu  and
      Zhang, Sheng  and
      Kwon, Gukyeong  and
      Perera, Pramuditha  and
      Zhu, Henghui  and
      Zhang, Yuhao  and
      Li, Alexander Hanbo  and
      Wang, William Yang  and
      Wang, Zhiguo  and
      Castelli, Vittorio  and
      Ng, Patrick  and
      Roth, Dan  and
      Xiang, Bing",
    editor = "Rogers, Anna  and
      Boyd-Graber, Jordan  and
      Okazaki, Naoaki",
    booktitle = "Findings of the Association for Computational Linguistics: ACL 2023",
    month = jul,
    year = "2023",
    address = "Toronto, Canada",
    publisher = "Association for Computational Linguistics",
    url = "https://aclanthology.org/2023.findings-acl.147",
    doi = "10.18653/v1/2023.findings-acl.147",
    pages = "2333--2346",
}

@incollection{mccloskey1989catastrophic,
  title={Catastrophic interference in connectionist networks: The sequential learning problem},
  author={McCloskey, Michael and Cohen, Neal J},
  booktitle={Psychology of learning and motivation},
  volume={24},
  pages={109--165},
  year={1989},
  publisher={Elsevier}
}

@inproceedings{qian2023decouple,
  title={Decouple before interact: Multi-modal prompt learning for continual visual question answering},
  author={Qian, Zi and Wang, Xin and Duan, Xuguang and Qin, Pengda and Li, Yuhong and Zhu, Wenwu},
  booktitle={Proceedings of the IEEE/CVF International Conference on Computer Vision},
  pages={2953--2962},
  year={2023}
}

@inproceedings{yuan2021tokens,
  title={Tokens-to-token vit: Training vision transformers from scratch on imagenet},
  author={Yuan, Li and Chen, Yunpeng and Wang, Tao and Yu, Weihao and Shi, Yujun and Jiang, Zi-Hang and Tay, Francis EH and Feng, Jiashi and Yan, Shuicheng},
  booktitle={Proceedings of the IEEE/CVF international conference on computer vision},
  pages={558--567},
  year={2021}
}

@article{Zhang2022OPTOP,
  title={OPT: Open Pre-trained Transformer Language Models},
  author={Susan Zhang and Stephen Roller and Naman Goyal and Mikel Artetxe and Moya Chen and Shuohui Chen and Christopher Dewan and Mona T. Diab and Xian Li and Xi Victoria Lin and Todor Mihaylov and Myle Ott and Sam Shleifer and Kurt Shuster and Daniel Simig and Punit Singh Koura and Anjali Sridhar and Tianlu Wang and Luke Zettlemoyer},
  journal={ArXiv},
  year={2022},
  volume={abs/2205.01068},
  url={https://api.semanticscholar.org/CorpusID:248496292}
}

@article{chung2024scaling,
  title={Scaling instruction-finetuned language models},
  author={Chung, Hyung Won and Hou, Le and Longpre, Shayne and Zoph, Barret and Tay, Yi and Fedus, William and Li, Yunxuan and Wang, Xuezhi and Dehghani, Mostafa and Brahma, Siddhartha and others},
  journal={Journal of Machine Learning Research},
  volume={25},
  number={70},
  pages={1--53},
  year={2024}
}

@article{zhao2024aligngpt,
  title={AlignGPT: Multi-modal Large Language Models with Adaptive Alignment Capability},
  author={Zhao, Fei and Pang, Taotian and Li, Chunhui and Wu, Zhen and Guo, Junjie and Xing, Shangyu and Dai, Xinyu},
  journal={arXiv preprint arXiv:2405.14129},
  year={2024}
}

@inproceedings{zhai2023investigating,
    title={Investigating the Catastrophic Forgetting in Multimodal Large Language Model Fine-Tuning},
    author={Yuexiang Zhai and Shengbang Tong and Xiao Li and Mu Cai and Qing Qu and Yong Jae Lee and Yi Ma},
    booktitle={Conference on Parsimony and Learning (Proceedings Track)},
    year={2023},
    url={https://openreview.net/forum?id=g7rMSiNtmA}
}

@inproceedings{karpathy2015deep,
  title={Deep visual-semantic alignments for generating image descriptions},
  author={Karpathy, Andrej and Fei-Fei, Li},
  booktitle={Proceedings of the IEEE conference on computer vision and pattern recognition},
  pages={3128--3137},
  year={2015}
}

@inproceedings{lin2014microsoft,
  title={Microsoft coco: Common objects in context},
  author={Lin, Tsung-Yi and Maire, Michael and Belongie, Serge and Hays, James and Perona, Pietro and Ramanan, Deva and Doll{\'a}r, Piotr and Zitnick, C Lawrence},
  booktitle={Computer Vision--ECCV 2014: 13th European Conference, Zurich, Switzerland, September 6-12, 2014, Proceedings, Part V 13},
  pages={740--755},
  year={2014},
  organization={Springer}
}

@inproceedings{goyal2017making,
  title={Making the v in vqa matter: Elevating the role of image understanding in visual question answering},
  author={Goyal, Yash and Khot, Tejas and Summers-Stay, Douglas and Batra, Dhruv and Parikh, Devi},
  booktitle={Proceedings of the IEEE conference on computer vision and pattern recognition},
  pages={6904--6913},
  year={2017}
}

@article{kirkpatrick2017overcoming,
  title={Overcoming catastrophic forgetting in neural networks},
  author={Kirkpatrick, James and Pascanu, Razvan and Rabinowitz, Neil and Veness, Joel and Desjardins, Guillaume and Rusu, Andrei A and Milan, Kieran and Quan, John and Ramalho, Tiago and Grabska-Barwinska, Agnieszka and others},
  journal={Proceedings of the national academy of sciences},
  volume={114},
  number={13},
  pages={3521--3526},
  year={2017},
  publisher={National Acad Sciences}
}

@article{sidorov2019textcaps,
    title={TextCaps: a Dataset for Image Captioning with Reading Comprehension},
    author={Sidorov, Oleksii and Hu, Ronghang and Rohrbach, Marcus and Singh, Amanpreet},
    journal={European Conference on Computer Vision},
    year={2020}
}

@inproceedings{singh2019towards,
    title={Towards VQA Models That Can Read},
    author={Singh, Amanpreet and Natarjan, Vivek and Shah, Meet and Jiang, Yu and Chen, Xinlei and Parikh, Devi and Rohrbach, Marcus},
    booktitle={Proceedings of the IEEE Conference on Computer Vision and Pattern Recognition},
    pages={8317-8326},
    year={2019}
}

@inproceedings{wang2022end,
  title={End-to-end transformer based model for image captioning},
  author={Wang, Yiyu and Xu, Jungang and Sun, Yingfei},
  booktitle={Proceedings of the AAAI conference on artificial intelligence},
  volume={36},
  number={3},
  pages={2585--2594},
  year={2022}
}

@article{li2021align,
  title={Align before fuse: Vision and language representation learning with momentum distillation},
  author={Li, Junnan and Selvaraju, Ramprasaath and Gotmare, Akhilesh and Joty, Shafiq and Xiong, Caiming and Hoi, Steven Chu Hong},
  journal={Advances in neural information processing systems},
  volume={34},
  pages={9694--9705},
  year={2021}
}

@article{dosovitskiy2020image,
  title={An image is worth 16x16 words: Transformers for image recognition at scale},
  author={Dosovitskiy, Alexey and Beyer, Lucas and Kolesnikov, Alexander and Weissenborn, Dirk and Zhai, Xiaohua and Unterthiner, Thomas and Dehghani, Mostafa and Minderer, Matthias and Heigold, Georg and Gelly, Sylvain and others},
  journal={arXiv preprint arXiv:2010.11929},
  year={2020}
}

@article{brown2020language,
  title={Language models are few-shot learners},
  author={Brown, Tom and Mann, Benjamin and Ryder, Nick and Subbiah, Melanie and Kaplan, Jared D and Dhariwal, Prafulla and Neelakantan, Arvind and Shyam, Pranav and Sastry, Girish and Askell, Amanda and others},
  journal={Advances in neural information processing systems},
  volume={33},
  pages={1877--1901},
  year={2020}
}

@article{li2017learning,
  title={Learning without forgetting},
  author={Li, Zhizhong and Hoiem, Derek},
  journal={IEEE transactions on pattern analysis and machine intelligence},
  volume={40},
  number={12},
  pages={2935--2947},
  year={2017},
  publisher={IEEE}
}

@article{lee2017overcoming,
  title={Overcoming catastrophic forgetting by incremental moment matching},
  author={Lee, Sang-Woo and Kim, Jin-Hwa and Jun, Jaehyun and Ha, Jung-Woo and Zhang, Byoung-Tak},
  journal={Advances in neural information processing systems},
  volume={30},
  year={2017}
}

@article{ahn2019uncertainty,
  title={Uncertainty-based continual learning with adaptive regularization},
  author={Ahn, Hongjoon and Cha, Sungmin and Lee, Donggyu and Moon, Taesup},
  journal={Advances in neural information processing systems},
  volume={32},
  year={2019}
}

@inproceedings{aljundi2018memory,
  title={Memory aware synapses: Learning what (not) to forget},
  author={Aljundi, Rahaf and Babiloni, Francesca and Elhoseiny, Mohamed and Rohrbach, Marcus and Tuytelaars, Tinne},
  booktitle={Proceedings of the European conference on computer vision (ECCV)},
  pages={139--154},
  year={2018}
}

@inproceedings{douillard2020podnet,
  title={Podnet: Pooled outputs distillation for small-tasks incremental learning},
  author={Douillard, Arthur and Cord, Matthieu and Ollion, Charles and Robert, Thomas and Valle, Eduardo},
  booktitle={Computer vision--ECCV 2020: 16th European conference, Glasgow, UK, August 23--28, 2020, proceedings, part XX 16},
  pages={86--102},
  year={2020},
  organization={Springer}
}

@inproceedings{rebuffi2017icarl,
  title={icarl: Incremental classifier and representation learning},
  author={Rebuffi, Sylvestre-Alvise and Kolesnikov, Alexander and Sperl, Georg and Lampert, Christoph H},
  booktitle={Proceedings of the IEEE conference on Computer Vision and Pattern Recognition},
  pages={2001--2010},
  year={2017}
}

@inproceedings{yan2021dynamically,
  title={Der: Dynamically expandable representation for class incremental learning},
  author={Yan, Shipeng and Xie, Jiangwei and He, Xuming},
  booktitle={Proceedings of the IEEE/CVF conference on computer vision and pattern recognition},
  pages={3014--3023},
  year={2021}
}

@inproceedings{mallya2018packnet,
  title={Packnet: Adding multiple tasks to a single network by iterative pruning},
  author={Mallya, Arun and Lazebnik, Svetlana},
  booktitle={Proceedings of the IEEE conference on Computer Vision and Pattern Recognition},
  pages={7765--7773},
  year={2018}
}

@inproceedings{douillard2022dytox,
  title={Dytox: Transformers for continual learning with dynamic token expansion},
  author={Douillard, Arthur and Ram{\'e}, Alexandre and Couairon, Guillaume and Cord, Matthieu},
  booktitle={Proceedings of the IEEE/CVF conference on computer vision and pattern recognition},
  pages={9285--9295},
  year={2022}
}

@article{wang2022foster,
  title={FOSTER: Feature Boosting and Compression for Class-Incremental Learning},
  author={Wang, Fu-Yun and Zhou, Da-Wei and Ye, Han-Jia and Zhan, De-Chuan},
  journal={arXiv preprint arXiv:2204.04662},
  year={2022}
}

@inproceedings{wang2022beef,
  title={Beef: Bi-compatible class-incremental learning via energy-based expansion and fusion},
  author={Wang, Fu-Yun and Zhou, Da-Wei and Liu, Liu and Ye, Han-Jia and Bian, Yatao and Zhan, De-Chuan and Zhao, Peilin},
  booktitle={The eleventh international conference on learning representations},
  year={2022}
}

@article{fernando2017pathnet,
  title={Pathnet: Evolution channels gradient descent in super neural networks},
  author={Fernando, Chrisantha and Banarse, Dylan and Blundell, Charles and Zwols, Yori and Ha, David and Rusu, Andrei A and Pritzel, Alexander and Wierstra, Daan},
  journal={arXiv preprint arXiv:1701.08734},
  year={2017}
}

@article{wang2022s,
  title={S-prompts learning with pre-trained transformers: An occam’s razor for domain incremental learning},
  author={Wang, Yabin and Huang, Zhiwu and Hong, Xiaopeng},
  journal={Advances in Neural Information Processing Systems},
  volume={35},
  pages={5682--5695},
  year={2022}
}

@inproceedings{wang2022dualprompt,
  title={Dualprompt: Complementary prompting for rehearsal-free continual learning},
  author={Wang, Zifeng and Zhang, Zizhao and Ebrahimi, Sayna and Sun, Ruoxi and Zhang, Han and Lee, Chen-Yu and Ren, Xiaoqi and Su, Guolong and Perot, Vincent and Dy, Jennifer and others},
  booktitle={European conference on computer vision},
  pages={631--648},
  year={2022},
  organization={Springer}
}

@inproceedings{wang2022learning,
  title={Learning to prompt for continual learning},
  author={Wang, Zifeng and Zhang, Zizhao and Lee, Chen-Yu and Zhang, Han and Sun, Ruoxi and Ren, Xiaoqi and Su, Guolong and Perot, Vincent and Dy, Jennifer and Pfister, Tomas},
  booktitle={Proceedings of the IEEE/CVF conference on computer vision and pattern recognition},
  pages={139--149},
  year={2022}
}

@inproceedings{smith2023coda,
  title={Coda-prompt: Continual decomposed attention-based prompting for rehearsal-free continual learning},
  author={Smith, James Seale and Karlinsky, Leonid and Gutta, Vyshnavi and Cascante-Bonilla, Paola and Kim, Donghyun and Arbelle, Assaf and Panda, Rameswar and Feris, Rogerio and Kira, Zsolt},
  booktitle={Proceedings of the IEEE/CVF conference on computer vision and pattern recognition},
  pages={11909--11919},
  year={2023}
}

@article{chen2025coin,
  title={CoIN: A Benchmark of Continual Instruction Tuning for Multimodel Large Language Models},
  author={Chen, Cheng and Zhu, Junchen and Luo, Xu and Shen, Hengtao and Song, Jingkuan and Gao, Lianli},
  journal={Advances in Neural Information Processing Systems},
  volume={37},
  pages={57817--57840},
  year={2025}
}

@article{greco2019psycholinguistics,
  title={Psycholinguistics meets continual learning: Measuring catastrophic forgetting in visual question answering},
  author={Greco, Claudio and Plank, Barbara and Fern{\'a}ndez, Raquel and Bernardi, Raffaella},
  journal={arXiv preprint arXiv:1906.04229},
  year={2019}
}

@inproceedings{liu2024moe,
  title={When MOE Meets LLMs: Parameter Efficient Fine-tuning for Multi-task Medical Applications},
  author={Liu, Qidong and Wu, Xian and Zhao, Xiangyu and Zhu, Yuanshao and Xu, Derong and Tian, Feng and Zheng, Yefeng},
  booktitle={Proceedings of the 47th International ACM SIGIR Conference on Research and Development in Information Retrieval},
  pages={1104--1114},
  year={2024}
}

@article{lopez2017gradient,
  title={Gradient episodic memory for continual learning},
  author={Lopez-Paz, David and Ranzato, Marc'Aurelio},
  journal={Advances in neural information processing systems},
  volume={30},
  year={2017}
}

@inproceedings{fang2023eva,
  title={Eva: Exploring the limits of masked visual representation learning at scale},
  author={Fang, Yuxin and Wang, Wen and Xie, Binhui and Sun, Quan and Wu, Ledell and Wang, Xinggang and Huang, Tiejun and Wang, Xinlong and Cao, Yue},
  booktitle={Proceedings of the IEEE/CVF conference on computer vision and pattern recognition},
  pages={19358--19369},
  year={2023}
}

@article{zhang2022opt,
  title={Opt: Open pre-trained transformer language models},
  author={Zhang, Susan and Roller, Stephen and Goyal, Naman and Artetxe, Mikel and Chen, Moya and Chen, Shuohui and Dewan, Christopher and Diab, Mona and Li, Xian and Lin, Xi Victoria and others},
  journal={arXiv preprint arXiv:2205.01068},
  year={2022}
}

@article{vaswani2017attention,
  title={Attention is all you need},
  author={Vaswani, Ashish and Shazeer, Noam and Parmar, Niki and Uszkoreit, Jakob and Jones, Llion and Gomez, Aidan N and Kaiser, {\L}ukasz and Polosukhin, Illia},
  journal={Advances in neural information processing systems},
  volume={30},
  year={2017}
}

@article{liu2023visual,
  title={Visual instruction tuning},
  author={Liu, Haotian and Li, Chunyuan and Wu, Qingyang and Lee, Yong Jae},
  journal={Advances in neural information processing systems},
  volume={36},
  pages={34892--34916},
  year={2023}
}

@article{bai2025qwen2,
  title={Qwen2. 5-vl technical report},
  author={Bai, Shuai and Chen, Keqin and Liu, Xuejing and Wang, Jialin and Ge, Wenbin and Song, Sibo and Dang, Kai and Wang, Peng and Wang, Shijie and Tang, Jun and others},
  journal={arXiv preprint arXiv:2502.13923},
  year={2025}
}

@article{cao2024continual,
  title={Continual llava: Continual instruction tuning in large vision-language models},
  author={Cao, Meng and Liu, Yuyang and Liu, Yingfei and Wang, Tiancai and Dong, Jiahua and Ding, Henghui and Zhang, Xiangyu and Reid, Ian and Liang, Xiaodan},
  journal={arXiv preprint arXiv:2411.02564},
  year={2024}
}

@inproceedings{zeng2025modalprompt,
  title={Modalprompt: Towards efficient multimodal continual instruction tuning with dual-modality guided prompt},
  author={Zeng, Fanhu and Zhu, Fei and Guo, Haiyang and Zhang, Xu-Yao and Liu, Cheng-Lin},
  booktitle={Proceedings of the 2025 Conference on Empirical Methods in Natural Language Processing},
  pages={12137--12152},
  year={2025}
}

@article{guo2025hide,
  title={Hide-llava: Hierarchical decoupling for continual instruction tuning of multimodal large language model},
  author={Guo, Haiyang and Zeng, Fanhu and Xiang, Ziwei and Zhu, Fei and Wang, Da-Han and Zhang, Xu-Yao and Liu, Cheng-Lin},
  journal={arXiv preprint arXiv:2503.12941},
  year={2025}
}

@misc{vicuna2023,
    title = {Vicuna: An Open-Source Chatbot Impressing GPT-4 with 90\%* ChatGPT Quality},
    url = {https://lmsys.org/blog/2023-03-30-vicuna/},
    author = {Chiang, Wei-Lin and Li, Zhuohan and Lin, Zi and Sheng, Ying and Wu, Zhanghao and Zhang, Hao and Zheng, Lianmin and Zhuang, Siyuan and Zhuang, Yonghao and Gonzalez, Joseph E. and Stoica, Ion and Xing, Eric P.},
    month = {March},
    year = {2023}
}

@article{kong2025yooop,
  title={YoooP: You Only Optimize One Prototype per Class for Non-Exemplar Incremental Learning},
  author={Kong, Jiangtao and Zong, Zhenyu and Zhou, Tianyi and Shao, Huajie},
  journal={Transactions on Machine Learning Research},
  year={2025}
}
\bibliographystyle{icml2026}

\newpage
\clearpage

\appendix
\setcounter{page}{1}
\onecolumn
\section*{Appendix}
\section{Notation List}\label{app:notation}
Table \ref{tab:notations} describes notations used in our work.
\newcolumntype{b}{>{\hsize=1.4\hsize}X}
\newcolumntype{s}{>{\hsize=.6\hsize}X}
\begin{table*}[!h]
    \caption{Summary of Notations.}\label{tab:notations}
    \centering
    \begin{tabularx}{\linewidth}{|s|b|}
        \hline
        Notation & Definition \\ \hline
        $t$ & The current task ID \\ 
        \hline
        $\mathcal{D}_t=\{(X_t, P_t, S_t)\}$ & The dataset for task-$t$ \\ 
        \hline
        $X_t = \{ x_{t,i} \}$ & The set of input images for task-$t$ \\ 
        \hline
        $P_t = \{ p_{t,i} \}$ & The set of associated text inputs (prompts or questions) for task-$t$ \\ 
        \hline
        $S_t = \{ s_{t,i} \}$ & The set of ground-truth sentences for each image-text pair \\ 
        \hline
        $n_t$ & The number of instances in task-$t$ \\ 
        \hline
        $\hat{s}_{t,i}$ & The predicted sentence for image $x_{t,i}$ \\ 
        \hline
        $\theta_t$ & Trainable parameters of the visual encoder (if fine-tuned) \\ 
        \hline
        $\theta_{\star}$ & Frozen parameters of the pre-trained visual encoder \\ 
        \hline
        $\phi_t$ & Trainable parameters of the Large Language Model (LLM) at task-$t$ (if fine-tuned) \\ 
        \hline
        $\phi_{\star}$ & Frozen parameters of the pre-trained Large Language Model (LLM) \\ 
        \hline
        $f_{\theta}(\cdot)$ & Visual encoder with parameters $\theta$ \\ 
        \hline
        $A_{\omega_t, Q_t}(\cdot)$ & Alignment module (Q-Former) with trainable parameters $\omega_t$ and learnable queries $Q_t$ \\ 
        \hline
        $g_{\phi}(\cdot)$ & Large Language Model (LLM) with parameters $\phi$ \\ 
        \hline
        $\omega_t$ & Trainable parameters of the Q-Former, including attention layers and a fully connected layer for query projection \\ 
        \hline
        $d_v$ & Dimension of the visual embeddings \\ 
        \hline
        $Q_t$ & Learnable query tokens \\ 
        \hline
        $n_Q$ & The number of learnable query tokens in Q-Former \\ 
        \hline
        $d_Q$ & The dimension of each query token in Q-Former \\ 
        \hline
        $v_t $ & Task-specific query tokens in the Mixture of Query (MoQ) module for task-$t$\\ 
        \hline
        $k_t$ & Task-specific key in the Mixture of Query (MoQ) module for task-$t$\\ 
        \hline
        $e_{t,i}=f_{\theta_{\star}}(x_{t,i})$ & Image embeddings \\ 
        \hline
        $\overline{e}_{t,i}$ & Average embedding over all patches in $e_{t,i}$ \\ 
        \hline
        $e_k$ & The $k$-th compact visual embedding in the dictionary \\ 
        \hline
        $K_t = [k_1, k_2, \dots, k_t]$ & Collected keys for all tasks \\ 
        \hline
        $V_t = [v_1, v_2, \dots, v_t]$ & Collected query tokens for all tasks \\ 
        \hline
        $S(\omega_t)$ & FIM-based metric for measuring conflict in parameter updates \\ 
        \hline
        $I_+(\omega_t)$ & Incremental contribution to $\Delta \mathcal{L}_{\mathcal{D}_t}$ \\ 
        \hline
        $I_-(\omega_t)$ & Decremental contribution to $\Delta \mathcal{L}_{\mathcal{D}_t}$ \\ 
        \hline
        $\Omega_{t} = \{ \omega_{t}, K_{t}, V_{t} \}$ & Learnable Parameters for task-$t$ \\
        \hline
        $D\in \mathbb{R}^{m\times d_v}$ & Overcomplete dictionary of compressed visual embeddings \\ 
        \hline
        $\alpha_k$ & Sparse coding coefficients \\ 
        \hline
        $|| \cdot ||_F$ & Frobenius norm \\ 
        \hline
        $m$ & Number of dictionary atoms in the embedding dictionary \\ 
        \hline
        $\gamma$ & Regularization weight in sparse coding \\ 
        \hline
    \end{tabularx}
\end{table*}

\twocolumn
\section{New Benchmark}\label{app:benchmark}
To evaluate the proposed method under a more realistic scenario, we propose a new setting for IL in OpenITG tasks. Unlike previous benchmarks~\citep{del2020ratt,qian2023decouple} that group images into five disjoint categories and remove images with objects from multiple categories, as we described at Sec.~\ref{sec:intro}, we split the tasks based on the image's ``main topic''. We define an image's main topic as the semantic category of its most prominent object. In our scenario, grouping images by main topic reflects real-world conditions, where a single image may predominantly contain one semantic category while also including other contextual elements. Thus, as shown in Fig.~\ref{fig:dataset_dis}, the distribution of semantic categories changes as the main topic evolves. This design effectively simulates the dynamic distribution shifts caused by environmental and temporal changes.

Based on this setting, we construct four new IL benchmarks for OpenITG. For the COCO Caption and VQAv2 tasks based on the COCO ImageCaption~\citep{lin2014microsoft} dataset, we label our benchmarks as ToS-COCO Caption and ToS-VQAv2. We first extract the area, class, and super-category of each foreground object from MSCOCO instance labels. The super-category of the object with the largest area is assigned as the image's main topic. Initially, 12 super-categories are obtained: People, Animal, Vehicle, Outdoor, Sports, Kitchen, Food, Furniture, Electronic, Appliance, Indoor, and Accessory. However, some main topics, such as ``People'' and ``Kitchen'' frequently appear across multiple main topics and are considered ``common topics.'' Images labeled with these common topics are reassigned based on the next largest non-common object. The final benchmark comprises 10 main topics: ``Animal'', ``Vehicle'', ``Outdoor'', ``Sports'', ``Food'', ``Furniture'', ``Electronics'', ``Appliances'', ``Indoor'', ``Accessories.'' 
For images without instance information, we employ GPT-4o with specific instructions to detect prominent foreground objects and assign classes and main topics accordingly.

For the TextCaps and TextVQA tasks, we follow a similar procedure and label our benchmarks as ToS-TextCaps and ToS-TextVQA. Since TextCaps provides only object classes and may include inaccuracies, we use GPT-4o with specific prompts to verify and detect the prominent foreground objects. Following the 10 main topics defined in MSCOCO, we assign each image a main topic. In our experiments, we observe that very few images have ``Animal'' as their main topic, so the final benchmark comprises 9 main topics, consistent with MSCOCO except for ``Animal.''

We estimate the GPT-4o-assisted labeling error rate by comparing the initial GPT-4o-assisted main-topic proposals with the final labels verified by two human annotators. For COCO-family images without instance information, GPT-4o was used to detect prominent foreground objects and assign classes and main topics. For TextCaps and TextVQA, GPT-4o was used to verify and detect prominent foreground objects. All final labels were then verified by two human annotators. In the shared COCO image pool, 1,960 (1.59\%) out of 123,287 images required GPT-4o-assisted initial proposals, and final human verification changed the main topic for 225 of them, giving an error rate of 11.48\%. In the shared TextCaps and TextVQA image pool, final human verification changed the main topic for 2,137 out of 25,119 images, giving an 8.51\% error rate.

\begin{figure}[!htb]
\centering
\begin{subfigure}{\linewidth}
    \centering
    \includegraphics[width=\linewidth]{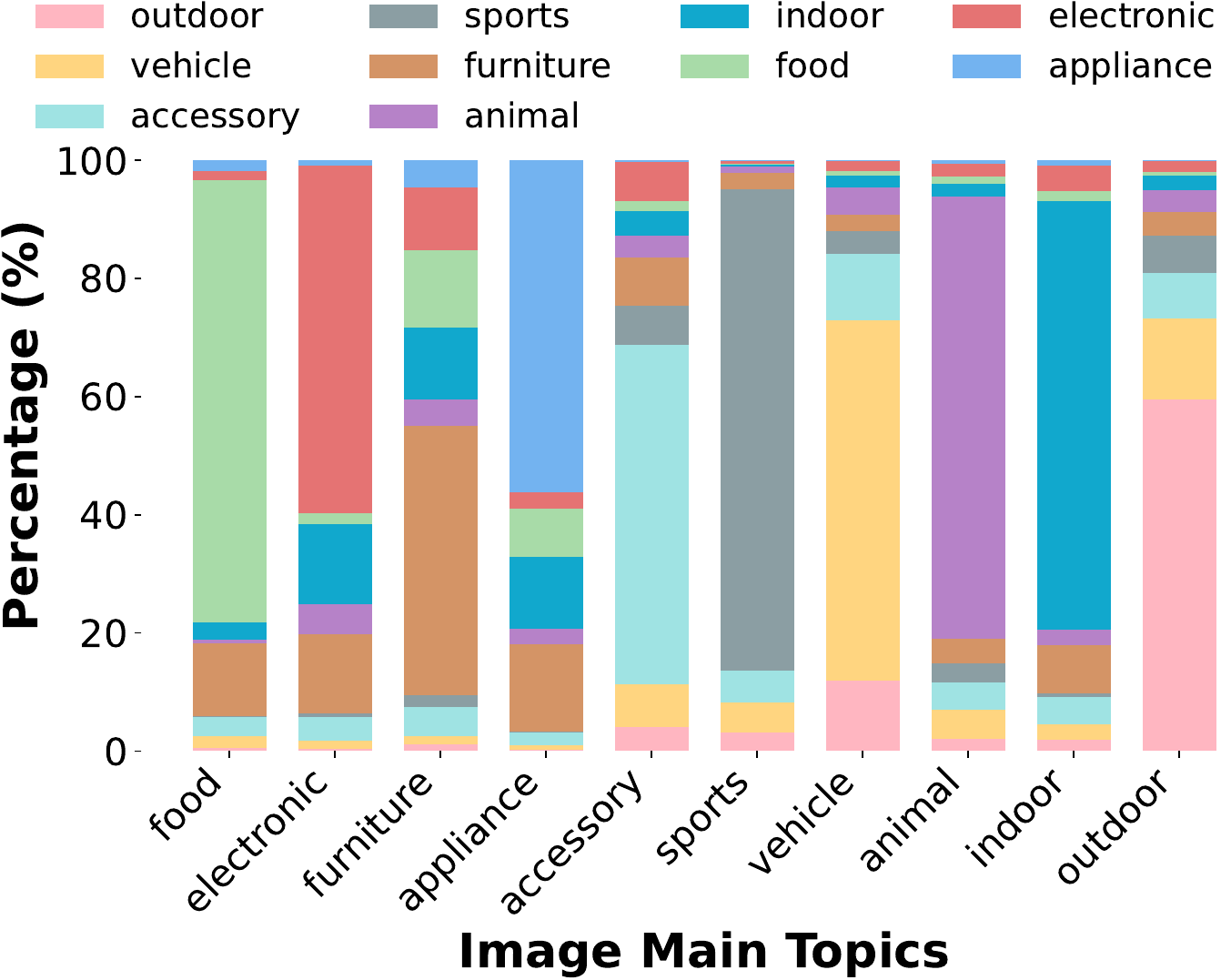}
    \label{fig:mscoco_dis}
\end{subfigure}
\hfill
\begin{subfigure}{\linewidth}
    \centering
    \includegraphics[width=\linewidth]{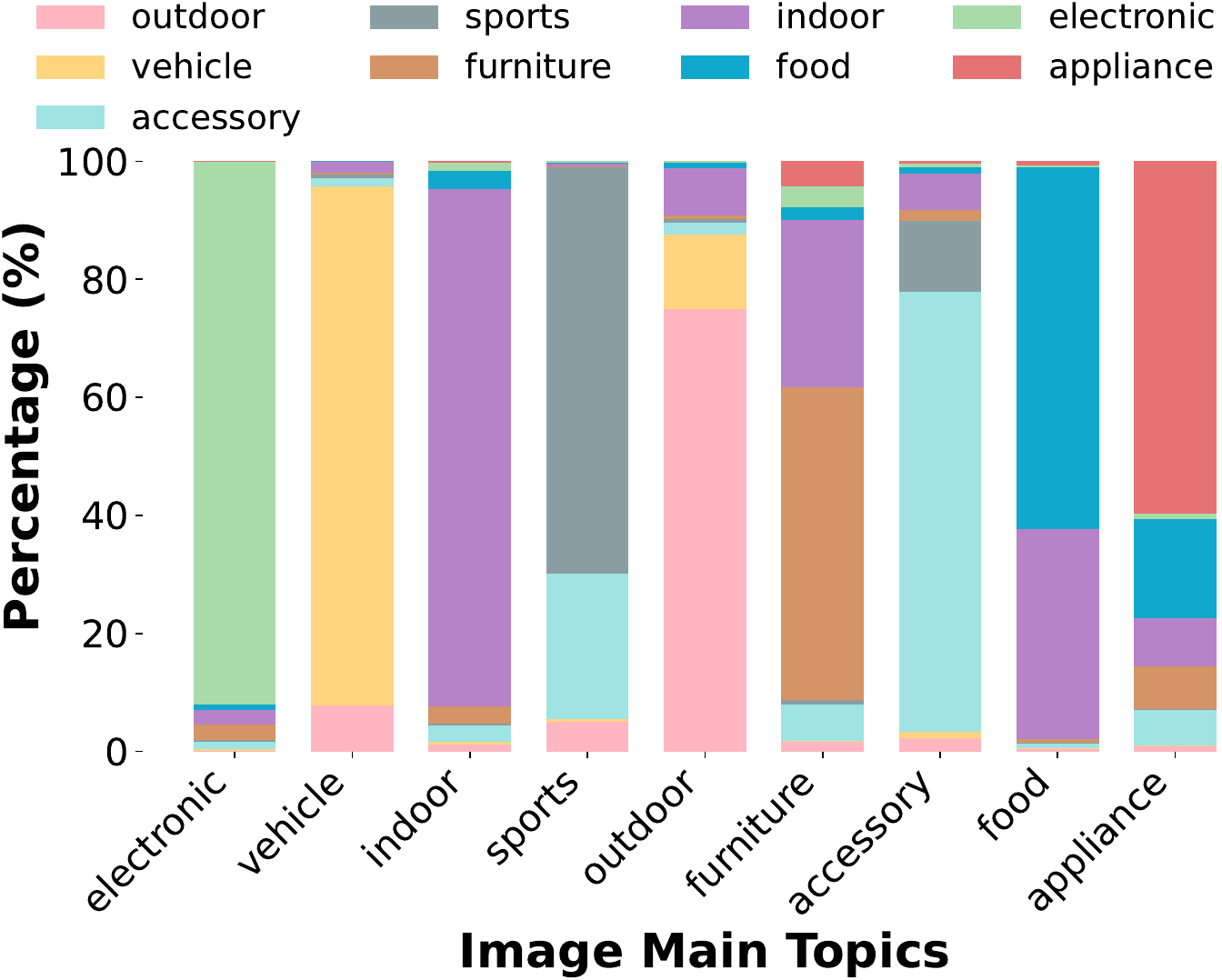}
    \label{fig:textcap_dis}
\end{subfigure}
\caption{The distribution of topics in four benchmarks. 
\textbf{Left:} ToS-COCO Caption \& ToS-VQAv2. 
\textbf{Right:} ToS-TextCaps \& ToS-TextVQA.}
\label{fig:dataset_dis}
\end{figure}

For ToS-COCO Caption, we adopt the train, validation, and test splits from prior work~\citep{karpathy2015deep} on COCO Caption and then partition the data into 10 tasks based on our labeled main topics. For ToS-VQAv2, ToS-TextCaps, and ToS-TextVQA, we use the official train-validation splits of VQAv2, TextCaps, and TextVQA and further divide them into 10, 9, and 9 tasks, respectively, according to our main topic labels.
\section{Experiments Protocol}\label{app:ep_protocol}
\begin{table*}[!t]
    \centering
    \caption{Hyper-parameters for fine-tuning ECA}
    \label{tab:hyperpm}
    \begin{adjustbox}{width=1\linewidth}
    \begin{tabular}{l|c|c|c|c}
        \toprule
        \textbf{Datasets} & \textbf{ToS-COCO Caption} & \textbf{ToS-VQAv2} & \textbf{ToS-TextCaps} & \textbf{ToS-TextVQA} \\
        \midrule
        Fine-tuning epochs & 5 & 5 & 5 & 5 \\
        Warm-up steps per task & 100 & 100 & 100 & 100 \\
        Learning rate & 1e-05 & 1e-5 & 1e-5 & 1e-5 \\
        Batch size & 64 & 64 & 32 & 32 \\        
        AdamW $\beta$ & (0.9, 0.999) & (0.9, 0.999) & (0.9, 0.999) & (0.9, 0.999) \\
        Weight decay & 0.05 & 0.05 & 0.05 & 0.05 \\
        Drop path & 0 & 0 & 0 & 0 \\
        Image resolution & 364 & 490 & 364 & 490 \\
        Inference beam size & 5 & 5 & 5 & 5 \\
        Prompt & a photo of & Question: \{\} Short answer: & \makecell[c]{Based on OCR: \{\}.\\ A photo of} & \makecell[c]{Based on OCR: \{\}.\\ Question: \{\} Short answer:} \\
        Atom Number $m$  & 7040 & 7040 & 7040 & 7040 \\
        $\gamma$  & 1 & 1 & 1 & 1 \\
        \bottomrule
    \end{tabular}
    \end{adjustbox}
\end{table*}

Following the common protocol of the OpenITG tasks~\citep{li2023blip,del2020ratt,antol2015vqa}, we use metrics BLEU@4, CIDEr, and SPICE for Image Captioning tasks, and VQA Accuracy for open-ended VQA. 
We assess IL performance using three metrics: \emph{Average Performance} (Avg), \emph{Forward Transfer} (FWT), and \emph{Backward Transfer} (BWT), following the protocol in~\citep{lopez2017gradient}.
Let $r_{t,\tau}$ denote the OpenITG metric value of task-$\tau$ after training on task-$t$, $\overline{b}_\tau$ the OpenITG metric value for task-$\tau$ on the initial model parameters, and $T$ the total task number. Then, we can obtain the Avg of the final task as $Avg=\frac{1}{T}\sum_{i=1}^Tr_{T,i}$, the BWT as $BWT=\frac{1}{T-1}\sum_{i=1}^{T-1}r_{T,i}-r_{i,i}$, and the FWT as $FWT=\frac{1}{T-1}\sum_{i=2}^{T}r_{i-1,i}-\overline{b}_i$. We report these metrics to capture the model's overall performance on all tasks, the effect of learning new tasks on past performance, and the model's ability to transfer knowledge to unseen tasks, respectively.

\section{Experimental Details}\label{app:ep_detail}
\textbf{Training Details.} 
In our experiments, we use the pre-trained BLIP-2~\citep{li2023blip}, which includes pre-trained ViT-g/14 from EVA-CLIP~\citep{fang2023eva} as the frozen visual encoder, an unsupervised-trained OPT-2.7B~\citep{zhang2022opt} as the frozen LLM, and a pre-trained Q-Former. Then we instantiate ECA and all other methods at the Q-Former from the pre-trained BLIP-2, and we compare their performance across different datasets.
For all approaches, we use the same optimizer hyper-parameters as in the original BLIP-2. 

To apply CODA-Prompt and Dual-Prompt, we follow the original works~\citep{smith2023coda,wang2022dualprompt} and insert deep prompts into the self-attention layers of the Q-Former. Specifically, for CODA-Prompt, prompts are inserted into layers 1–5 of the Q-Former, matching the depth of both the Q-Former and ViT-B/16~\citep{dosovitskiy2020image}. The prompt length is set to $8$. Based on the ablation studies in the original papers, we increase the prompt pool size to $30$ prompts per task (300 in total for 10 tasks) for ToS-COCO Caption and $50$ prompts per task (500 in total for 10 tasks) for the rest of the datasets. This adjustment is made to ensure parameter efficiency, fairness, and to enhance performance. The orthogonal loss weight is set to $0.1$ following the original configuration. For Dual-Prompt, we insert G-prompts into layers 1–2 and E-prompts into layers 3–5. Similarly, we increase both G-prompt and E-prompt length to $500$ for ToS-COCO Caption and $800$ for the rest of the datasets. The balance factor in Dual prompt is set to $1$ following the original configuration.

For our ECA, within the Q-Former, FeDEx inserts Parallel Adapters (PAs) in every self-attention and feed-forward layer. The low-rank for PAs is set to $30$ for attention layers, and $512$ for feed-forward layers, with all PA scales set to $4$ according to~\citep{he2022towards}. For computational efficiency, we replace the non-linear activation with an identity function so that multiple frozen PAs can be merged during training and inference. For the embedding dictionary, we set the number of atoms as $m=5\times d_v$, where $d_v$ matches the dimension of the frozen visual encoder $f(x_{t,i};\theta_{\star})$. The weight of DR loss, $\lambda$ in Eq.~\ref{eq:final_loss}, is set to $0.1$, and the threshold of $S(\omega_t)$ in Eq.~\ref{eq:metric} is set to $0.5$ for all experiments.

For MoE-LoRA~\citep{chen2025coin}, we follow the original paper by setting 8 experts per layer and inserting MoE-LoRA into every feed-forward network, with each expert having a low-rank of 512, consistent with our ECA configuration.

For LwF and EWC~\citep{lee2017overcoming}, we directly apply them into the backbone, ``Vanilla (PA),'' to report the performance for fair comparison.

During training, we extract image features as in BLIP-2 and provide additional tokens to the Q-Former. Specifically, following Q-Former inputs, we input the question token for ToS-VQAv2, the official OCR token for ToS-TextCaps, and both OCR and question tokens for ToS-TextVQA to enable interaction with query tokens via the self-attention layers. These tokens introduce a small number of trainable parameters in the embedding layer and feed-forward network inside the Q-Former, so the reported number of trainable parameters differs across datasets. Training hyper-parameters are detailed in Tab.~\ref{tab:hyperpm}.
\section{Derivation and Proof of Theorem~\ref{theo:conflict}}
\label{app:derivation}
Here we provide the details of the derivation and proof of Theorem~\ref{theo:conflict}.
\conflict*
\subsection{Derivation of Theorem~\ref{theo:conflict}}
First, we introduce how to obtain the incremental and decremental contributions for $\mathcal{L}_{\mathcal{D}_t}(\omega_t)$
Our goal is to measure how $\omega_t$, updated by Eq.~\ref{eq:ce_loss} on dataset $\mathcal{D}_{t+1}$ for task-$t+1$, affects the performance on $\mathcal{D}_t$ for task-$t$. To this end, we first obtain the loss on $\mathcal{D}_t$ as the expectation of the cross-entropy loss $\mathcal{L}_{ce}$ on dataset $\mathcal{D}_t$:
\[
\mathcal{L}_{\mathcal{D}_t} = \mathbb{E}_{\mathcal{D}_t}[\mathcal{L}_{ce}].
\]
Then, we perform a second-order Taylor expansion of the loss $\mathcal{L}_{\mathcal{D}_t}(\omega)$ around $\omega_t$:
\begin{equation}
\label{eq:deltaL_diagonal}
\begin{aligned}
\Delta \mathcal{L}_{\mathcal{D}_t}(\Delta \omega) 
&\triangleq \mathcal{L}_{\mathcal{D}_t}(\omega_t + \Delta\omega) - \mathcal{L}_{\mathcal{D}_t}(\omega_t)\\[1mm]
&\approx \sum_{i=1}^N \nabla_{\omega_i} \mathcal{L}_{\mathcal{D}_t}(\omega_t)\,\Delta\omega_i \\
&\quad + \frac{1}{2} \sum_{i=1}^N H^i_{\mathcal{D}_t}(\omega_t)\, (\Delta\omega_i)^2,
\end{aligned}
\end{equation}
where $N$ is the number of parameters in $\omega_t$, and $H^i_{\mathcal{D}_t}(\omega_t)$ denotes the $i$-th diagonal element of the Hessian of $\mathcal{L}_{\mathcal{D}_t}$. 

In theory, the full Hessian $H_{\mathcal{D}_t}(\omega_t)$ contains off-diagonal elements that capture parameter interactions. However, computing the full Hessian is computationally prohibitive, and in many practical scenarios the off-diagonal elements are relatively small compared to the diagonal. Therefore, it is common to approximate the Hessian by its diagonal. Moreover, under the negative log-likelihood loss (i.e., for $\mathcal{L}_{\mathcal{D}_t}(\cdot)$), the diagonal of the Hessian is well approximated by the diagonal of the Fisher Information Matrix (FIM)~\citep{kirkpatrick2017overcoming}. Hence, we set
\[
H^i_{\mathcal{D}_t}(\omega_t) \approx F^i_{\mathcal{D}_t}(\omega_t).
\]
Since we use the (empirical) Fisher approximation, $F^i_{\mathcal{D}_t}(\omega_t)\ge 0$ for all $i$.

Under a small-step gradient descent update on $\mathcal{D}_{t+1}$ with $1$ learning rate, a typical update for each parameter is given by
\[
\Delta\omega_i = -\nabla_{\omega_i} \mathcal{L}_{\mathcal{D}_{t+1}}(\omega_t).
\]
Substituting this into Eq.~\ref{eq:deltaL_diagonal} yields the per-parameter impact function:
\begin{equation}
\label{eq:I_omega}
\begin{aligned}
I(\omega_{t,i}) =\; & -\nabla_{\omega_i} \mathcal{L}_{\mathcal{D}_t}(\omega_t)\,\nabla_{\omega_i} \mathcal{L}_{\mathcal{D}_{t+1}}(\omega_t) \\
&\quad + \frac{1}{2}\,F^i_{\mathcal{D}_t}(\omega_t)\,\Bigl(\nabla_{\omega_i} \mathcal{L}_{\mathcal{D}_{t+1}}(\omega_t)\Bigr)^2.
\end{aligned}
\end{equation}
A large positive $I(\omega_{t,i})$ indicates that updating $\omega_{t,i}$ on $\mathcal{D}_{t+1}$ would significantly impair the alignment on $\mathcal{D}_t$.

Furthermore, to better capture the overall impact on the model parameters, we aggregate these per-parameter impacts by distinguishing the positive and negative contributions:
\begin{equation}
\label{eq:I_pos}
I_{+}(\omega_t)=\sum_{i=1}^N \max\{0,\,I(\omega_{t,i})\},
\end{equation}
\begin{equation}
\label{eq:I_neg}
I_{-}(\omega_t)=\sum_{i=1}^N \min\{0,\, I(\omega_{t,i})\}.
\end{equation}
Then we define the normalized conflict metric as
\begin{equation}
\label{eq:metric_appendix}
S(\omega_t) = \frac{I_{+}(\omega_t)}{I_{+}(\omega_t)+|I_{-}(\omega_t)|} \in [0,1].
\end{equation}
The metric $S(\omega_t)$ reflects the overall conflict between parameter updates for the new task and the preservation of prior knowledge.

\subsection{Proof of Theorem~\ref{theo:conflict}}



\begin{figure*}[!t]
\centering
\begin{subfigure}{0.30\textwidth}
\centering
\includegraphics[width=\linewidth]{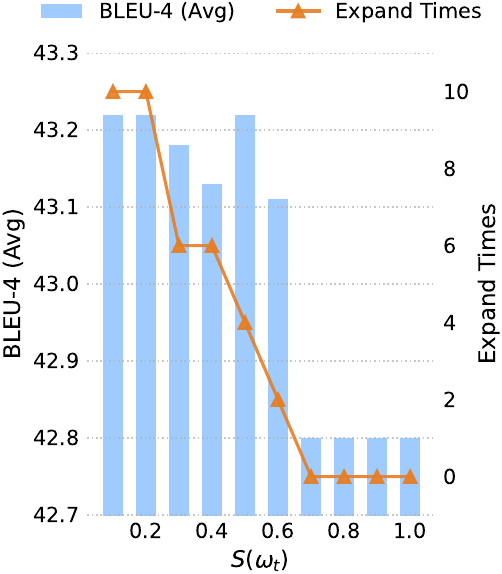}
\label{fig:thr_b4}
\end{subfigure}
\begin{subfigure}{0.31\textwidth}
\centering
\includegraphics[width=\linewidth]{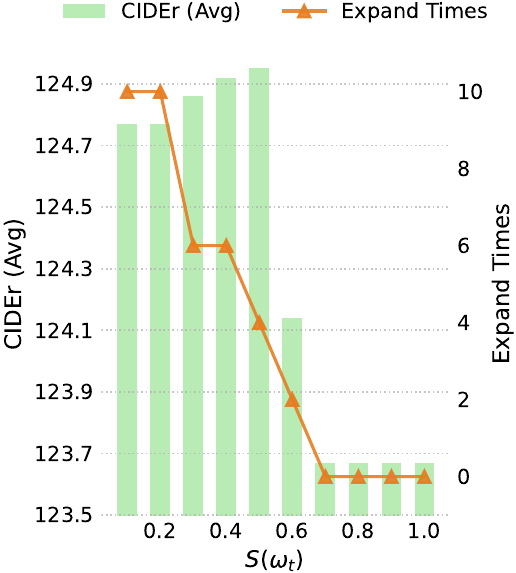}
\label{fig:thr_c}
\end{subfigure}
\begin{subfigure}{0.31\textwidth}
\centering
\includegraphics[width=\linewidth]{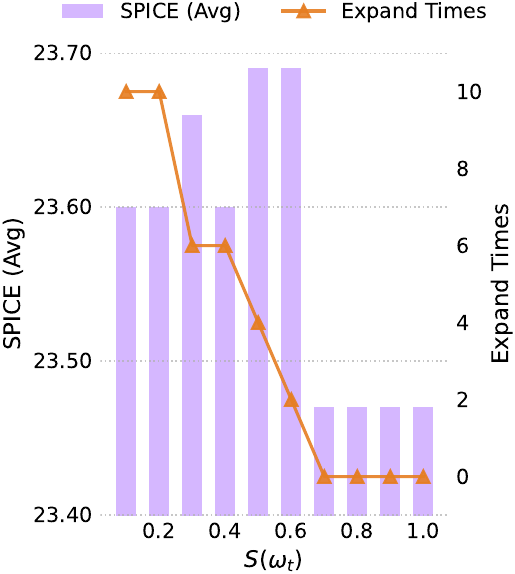}
\label{fig:thr_s}
\end{subfigure}
\begin{subfigure}{0.30\textwidth}
\centering
\includegraphics[width=\linewidth]{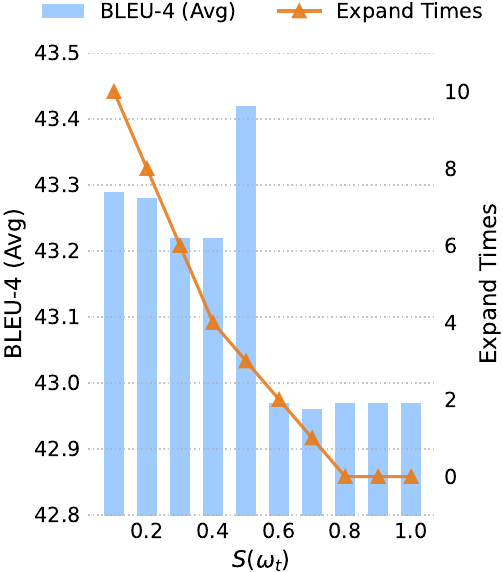}
\label{fig:eca_thr_b4}
\end{subfigure}
\begin{subfigure}{0.31\textwidth}
\centering
\includegraphics[width=\linewidth]{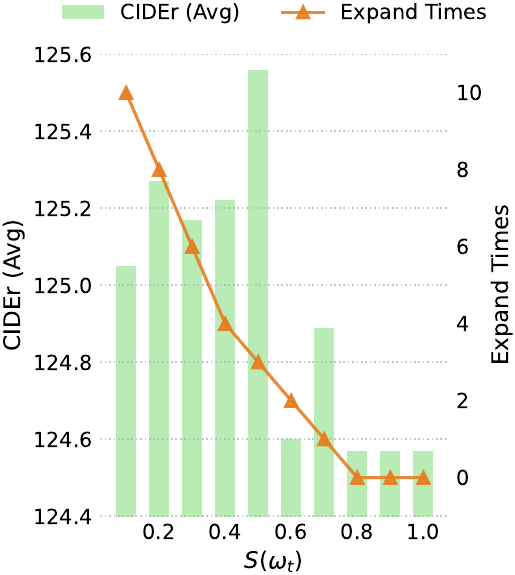}
\label{fig:eca_thr_c}
\end{subfigure}
\begin{subfigure}{0.31\textwidth}
\centering
\includegraphics[width=\linewidth]{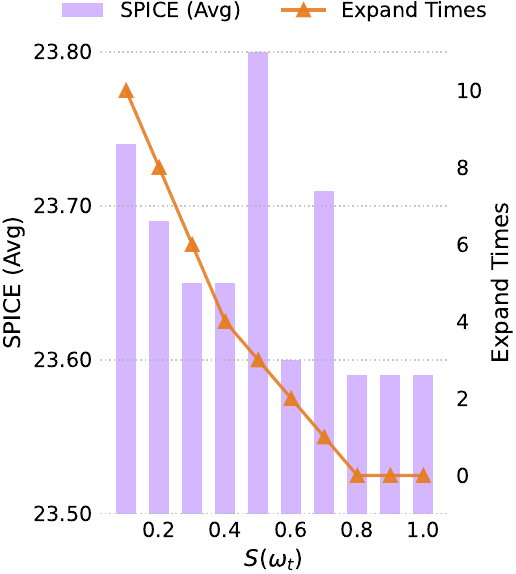}
\label{fig:eca_thr_s}
\end{subfigure}
\vspace{-0.15in}
\caption{Performance of different structures with different threshold of $S(\omega_t)$ in FeDEx. \textbf{Upper Line:} the performance of ``PA+MoQ+FeDEx''; \textbf{Bottom Line:} the performance of the whole ``ECA.''}
\label{fig:threshold}
\vspace{-0.1in}
\end{figure*}

\begin{proof}

By definition, we have:
\[
\Delta\mathcal{L}_{\mathcal{D}_t}(\Delta\omega)=I_+(\omega_t)+I_-(\omega_t).
\]

Since $I_+(\omega_t)\geq 0$ and $I_-(\omega_t)\leq 0$, we set $|I_-(\omega_t)|=-I_-(\omega_t)$ and rewrite the conflict metric as:
\[
S(\omega_t)=\frac{I_+(\omega_t)}{I_+(\omega_t)+|I_-(\omega_t)|}.
\]

Then, $\Delta\mathcal{L}_{\mathcal{D}_t}(\Delta\omega)\leq 0$ holds if and only if:
\[
I_+(\omega_t)\leq |I_-(\omega_t)|,
\]
which, after dividing by the positive term $(I_+(\omega_t)+|I_-(\omega_t)|)$, gives:
\[
S(\omega_t)\leq 0.5.
\]

Conversely, if $S(\omega_t)\leq 0.5$, then:
\[
I_+(\omega_t)\leq |I_-(\omega_t)|\Longrightarrow I_+(\omega_t)+I_-(\omega_t)\leq 0.
\]

Similarly, if $S(\omega_t)>0.5$, then $I_+(\omega_t)>|I_-(\omega_t)|$, which implies:
\[
\Delta\mathcal{L}_{\mathcal{D}_t}(\Delta\omega)=I_+(\omega_t)+I_-(\omega_t)>0.
\]

Thus, the equivalences hold in both directions, completing the proof.
\end{proof}

\subsection{Experiment on $S(\omega)$}\label{app:conflict_exp}
To verify our Theorem~\ref{theo:conflict}, 
we sweep the threshold of $S(\omega_t)$ and observe that $S(\omega_t)=0.5$ consistently attains the best trade-off across architectures and metrics (Fig.~\ref{fig:threshold}). 
Thresholds substantially smaller than $0.5$ tend to reduce performance due to unnecessary adapter expansion.

\section{Influence of Loss}\label{app:add-ablation}
We examine how each loss in MoQ influences model performance, \ie $\mathcal{L}_\text{orth}$ and $\mathcal{L}_\text{key}$ in Eq.~\ref{eq:moq_loss}, we compare their performance separately. As shown in Tab.~\ref{tab:ablation_moqloss}, the $\mathcal{L}_\text{orth}$ can significantly increase the BWT, which means it decreases the influence of the previously learned query tokens. Regarding the $\mathcal{L}_\text{key}$, optimizing only by $\mathcal{L}_\text{key}$ without the $\mathcal{L}_\text{orth}$ may increase the interference across tasks. However, while optimizing both losses, $\mathcal{L}_\text{key}$ can ensure that each task-specific key is relevant to visual embeddings in task-$t$, and leverage the $\mathcal{L}_\text{orth}$ to preserve distinct sets of query tokens.
\begin{table}[!th]
    \centering
    \caption{Ablation study on ToS-COCO Caption.}
    \label{tab:ablation_moqloss}
    \begin{adjustbox}{width=\linewidth}
    \setlength{\tabcolsep}{1pt}
    \begin{tabular}{l ccc ccc ccc}
        \toprule
        \multirow{2}{*}{\textbf{Method}} & 
        \multicolumn{3}{c}{\textbf{BLEU-4}} & 
        \multicolumn{3}{c}{\textbf{CIDEr}} & 
        \multicolumn{3}{c}{\textbf{SPICE}} \\
        \cmidrule(lr){2-4} \cmidrule(lr){5-7} \cmidrule(lr){8-10}
         & Avg $\uparrow$ & BWT $\uparrow$ & FWT $\uparrow$  
         & Avg $\uparrow$ & BWT $\uparrow$ & FWT $\uparrow$  
         & Avg $\uparrow$ & BWT $\uparrow$ & FWT $\uparrow$ \\
        \midrule
        Vanilla (PA)           & 42.70 & -1.49 & 6.48 & 123.00 & -4.50 & 19.15 & 23.39 & -0.78 & 2.64 \\
        PA+MoQ ($\mathcal{L}_\text{key}$)              & 42.76 & -1.80 & 6.66 & 123.29 & -4.34 & 18.50 & 23.37 & -0.79 & 2.71 \\
        PA+MoQ ($\mathcal{L}_\text{orth}$)           & 42.67 & -1.30 & 6.92 & 123.21 & -3.99 & 19.84 & 23.40 & -0.68 & 2.88 \\
        PA+MoQ ($\mathcal{L}_\text{orth}+\mathcal{L}_\text{key}$)   & 42.80 & -1.25 & 6.77 & 123.67 & -3.66 & 19.05 & 23.47 & -0.59 & 2.71 \\
        \bottomrule
    \end{tabular}
    \end{adjustbox}
\end{table}

\begin{table}[!th]
\centering
\begin{subtable}{\linewidth}
\centering
\caption{``PA+MoQ+DR'' with different $\lambda$}
\label{tab:lambda}
\begin{adjustbox}{width=\linewidth}
\begin{tabular}{l ccccc}
\toprule
$\lambda$ & 0.01 & 0.05 & 0.1 & 0.5 & 1 \\
\midrule
BLEU-4 (Avg) & 42.61 & 42.34 & \textbf{42.97} & 42.71 & 42.39 \\
CIDEr (Avg)  & 123.38 & 122.64 & \textbf{124.57} & 123.01 & 122.84 \\
SPICE (Avg)  & 23.54 & 23.29 & \textbf{23.59} & 23.59 & 23.49 \\
\bottomrule
\end{tabular}
\end{adjustbox}
\end{subtable}\hfill
\begin{subtable}{\linewidth}
\centering
\caption{``PA+MoQ+DR'' with different $m=M\times d_v$}
\label{tab:dr}
\begin{adjustbox}{width=\linewidth}
\begin{tabular}{l ccccc}
\toprule
$m$       & 2.5x & 5x & 7.5x & 10x & 12.5x \\
\midrule
BLEU-4 (Avg) & 41.63 & \textbf{42.97} & 42.65 & 42.56 & 42.85 \\
CIDEr (Avg)  & 122.42 & \textbf{124.57} & 123.22 & 123.38 & 123.59 \\
SPICE (Avg)  & 23.53 & \textbf{23.59} & 23.54 & 23.46 & 23.48 \\
\bottomrule
\end{tabular}
\end{adjustbox}
\end{subtable}
\caption{Ablations of Hyper-parameters in DR on ToS COCO Caption}
\label{tab:dr_ablate}
\end{table}

\begin{figure*}[!t]
\vspace{-0.05in}
\centering
\includegraphics[width=0.99\linewidth]{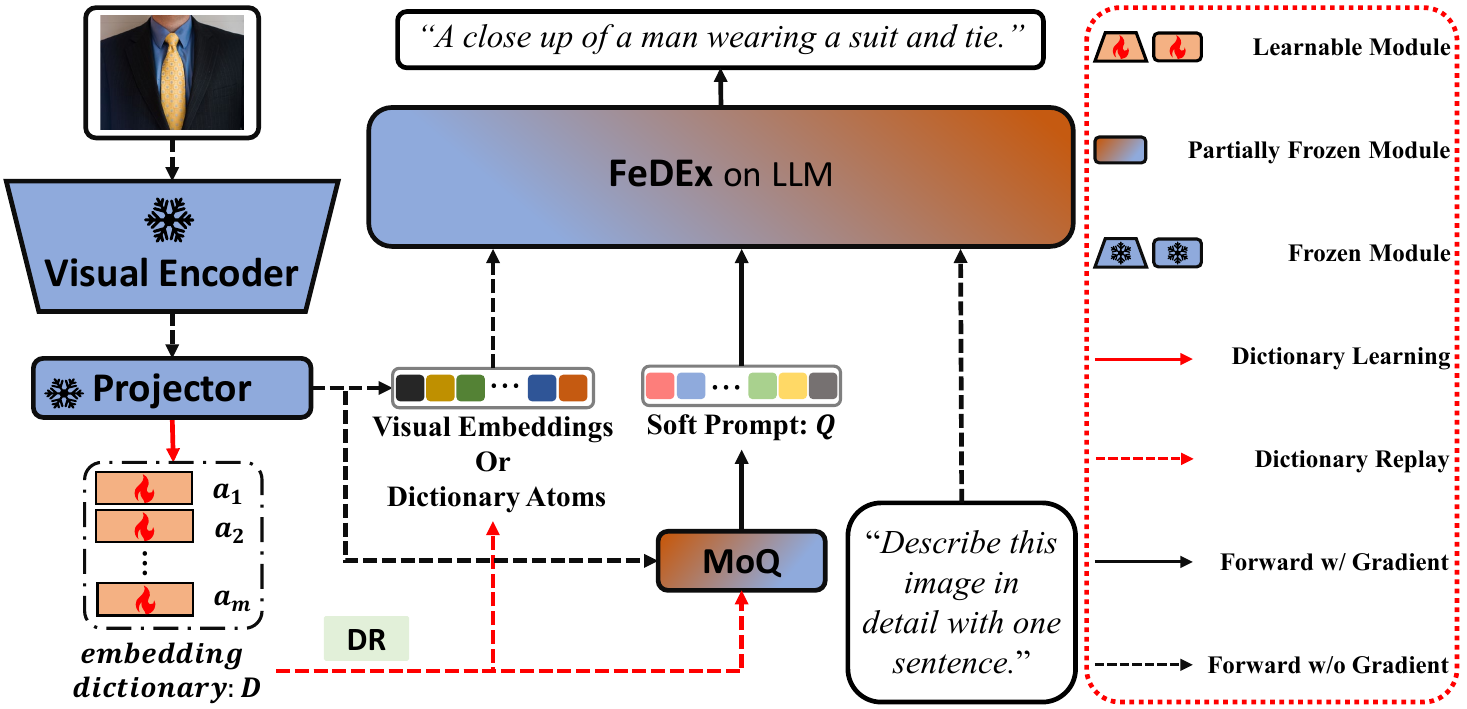}
\vspace{-0.05in}
\caption{The framework of our ECA instantiated on a projector-based multi-modal LLM (\eg, LLaVA). An input image is processed by a frozen visual encoder with a pre-trained projector to produce visual tokens. These visual tokens enter the \textbf{M}ixture \textbf{o}f \textbf{Q}uery module, which generates soft prompt tokens that are concatenated with visual tokens to interact with the LLM equipped with \textbf{F}ish\textbf{e}r \textbf{D}ynamic \textbf{Ex}pansion, to generate text conditioned on visual context. After the current task, visual tokens update the embedding dictionary via sparse dictionary learning. During training, the \textbf{D}ictionary \textbf{R}eplay module replays the embedding dictionary to retain the former alignment.}
\label{fig:llava_pipeline}
\end{figure*}

\section{Effect of Hyper-parameters}\label{app:hyper-parameter}
Then we apply the grid search to explore the effect of three hyper-parameters, \ie DR weight $\lambda$, DR's embedding dictionary atom number $m$, and the threshold of FIM-based metric value, $S(\omega_t)$, in FeDEx.
For exploring the hyper-parameters of DR, we tested a range of $\lambda$ and $m$ settings on the ``PA+MoQ+DR'' structure. 
As shown in Tab.~\ref{tab:lambda} and Tab.~\ref{tab:dr}, we set $\lambda=0.1$ in all experiments and $m=5\times d_v$.

\section{Parameter and Inference Efficiency}
\label{app:param-eff}
In this section, we further analyze the parameter and inference efficiency of methods on an NVIDIA A40 GPU. As shown in Tab.~\ref{tab:efficiency}, we compare parameter and runtime efficiency on the ToS-COCO Caption benchmark. All methods share the same pre-trained BLIP-2 backbone and are evaluated on the same GPU with an identical batch size and input length.
\begin{table*}[!th]
\centering
\caption{Parameter and inference efficiency analysis on ToS-COCO Caption.}
\vspace{-0.1in}
\label{tab:efficiency}
\begin{adjustbox}{width=\linewidth}
    \begin{tabular}{lccccccc|c}
        \toprule 
        \multirow{2}{*}{\textbf{Metric}} &\multicolumn{8}{c}{\textbf{Method}} \\
        \cmidrule(lr){2-9}
        &Vanilla (PA) &Vanilla (Q-Former) &LwF &EWC &Dual-Prompt &CODA-Prompt &MoE-LoRA &\textbf{ECA (Ours)} \\
        \midrule
        Trainable Params $\downarrow$ &12.29M &107.13M &12.29M &12.29M &14.30M &15.41M &98.84M &12.29M \\
        Training GPU Memory $\downarrow$ &18.80G &21.50G &18.80G &18.80G &19.47G &19.55G &21.37G &18.92G \\
        Inference GPU Memory $\downarrow$ &10.67G &10.65G &10.67G &10.67G &10.70G &10.70G &11.02G &10.72G \\
        Throughput (token/s) $\uparrow$ &36.52 &37.62 &36.52 &36.52 &32.82 &34.77 &36.37 &36.49 \\
        \bottomrule
    \end{tabular}
\end{adjustbox}
\end{table*}
Comparing the Trainable Params and training/inference GPU memory usage in Tab.~\ref{tab:efficiency}, ECA uses almost the same number of trainable parameters and peak GPU memory as baselines with single PA (\ie Vanilla (PA), LwF, EWC), while matching or even slightly exceeding their inference throughput. In contrast, methods with much larger parameter budgets (\ie Dual-Prompt, CODA-Prompt, MoE-LoRA) are slower and require more memory. These results support our claim that ECA enhances continuous performance with parameter efficiency and that it better leverages a limited alignment capacity rather than blindly expanding the model.
\section{Additional Evaluation under Other Protocol}\label{app:protocol}
To examine whether the gains of ECA are tied to our ToS main-topic split, we additionally evaluate ECA under the ConVS protocol on CL-VQA2.0 from TRIPLET~\cite{qian2023decouple}. Unlike our ToS protocol, which constructs tasks around dominant visual topics, ConVS provides an alternative protocol for Incremental Learning in OpenITG.

We use the same pre-trained BLIP-2 backbone and follow the same training setup as in Appendix~\ref{app:ep_detail} for all methods except Dual-Prompt. For fair comparison, we increase both G-prompt and E-prompt length to 1600 for ConVS on CL-VQA2.0 since the total task number is 5. 

\begin{table}[!th]
    \centering
    \caption{Evaluation under ConVS protocol on CL-VQA2.0. \textbf{Bold}: Best results on each dataset. \underline{Underline}: Second-best results on each dataset. ``Avg'': Final Average performance; ``BWT'': Backward Transfer; ``FWT'': Forward Transfer.}
    \vspace{-0.05in}
    \label{tab:convs}
    \begin{adjustbox}{width=\linewidth}
    \setlength{\tabcolsep}{1pt}
    \begin{tabular}{l c ccc}
        \toprule
        \multirow{2}{*}{\textbf{Method}} & 
        \multirow{2}{*}{\textbf{\makecell[c]{\# Trainable\\ Params}}} & 
        \multicolumn{3}{c}{\textbf{VQA Acc}} \\
        \cmidrule(lr){3-5}
        & & Avg $\uparrow$ & BWT $\uparrow$ & FWT $\uparrow$ \\
        \midrule
        ZeroShot & 0 M & 49.89 & - & - \\
        Vanilla (PA) & 21.74 M & 67.54 & -2.25 & 15.51 \\
        LwF~\citep{li2017learning} & 21.74 M & 69.71 & 0.59 & 18.45 \\
        EWC~\citep{lee2017overcoming} & 21.74 M & 63.79 & -0.92 & 15.43 \\
        Dual-Prompt~\citep{wang2022dualprompt} & 22.88 M & 65.76 & 0.9 & 13.09 \\
        CODA-Prompt~\citep{smith2023coda} & 24.37 M & 66.47 & \underline{1.04} & 13.84 \\
        MoE-LoRA~\citep{liu2024moe} & 195.71 M & \underline{70.68} & 1.02 & \underline{18.70} \\
        \midrule
        \textbf{ECA (Ours)} & 21.74 M & \textbf{71.03} & \textbf{1.13} & \textbf{19.13} \\
        \bottomrule
    \end{tabular}
    \end{adjustbox}
    \vspace{-0.1in}
\end{table}

As shown in Tab.~\ref{tab:convs}, ECA consistently outperforms the baselines under the ConVS protocol on CL-VQA2.0, which confirms that ECA’s advantage is not specific to our ToS benchmark, but also holds under an alternative continual-learning protocol. Together with the main ToS results, the result further supports ECA as a general exemplar-free IL approach for OpenITG.
\begin{table*}[!t]
    \centering
    \caption{Evaluation on ToS-TextCaps and ToS-TextVQA. \textbf{Bold}: Best results on each dataset. \underline{Underline}: Second-best results on each dataset. ``Avg'': Final Average performance; ``BWT'': Backward Transfer; ``FWT'': Forward Transfer.}
    \vspace{-0.05in}
    \label{tab:llava_text}
    \begin{adjustbox}{width=1\textwidth}
    \setlength{\tabcolsep}{1.2pt}
    \begin{tabular}{l c ccc ccc ccc ccc}
        \toprule
        \multicolumn{2}{l}{\textbf{Tasks}} & \multicolumn{9}{c}{\textbf{ToS-TextCaps}} & \multicolumn{3}{c}{\textbf{ToS-TextVQA}} \\ \cmidrule(lr){1-2} \cmidrule(lr){3-11} \cmidrule(lr){12-14}
        \multirow{2}{*}{\textbf{Method}} & 
        \multirow{2}{*}{\textbf{\makecell[c]{\# Trainable\\ Params}}} & 
        \multicolumn{3}{c}{\textbf{BLEU-4}} & 
        \multicolumn{3}{c}{\textbf{CIDEr}} & 
        \multicolumn{3}{c}{\textbf{SPICE}} & 
        \multicolumn{3}{c}{\textbf{VQA Acc}} \\
        \cmidrule(lr){3-5} \cmidrule(lr){6-8} \cmidrule(lr){9-11} \cmidrule(lr){12-14}
        & & Avg $\uparrow$ & BWT $\uparrow$ & FWT $\uparrow$ 
        & Avg $\uparrow$ & BWT $\uparrow$ & FWT $\uparrow$ 
        & Avg $\uparrow$ & BWT $\uparrow$ & FWT $\uparrow$ 
        & Avg $\uparrow$ & BWT $\uparrow$ & FWT $\uparrow$ \\
        \midrule
        ZeroShot & 0 M & 10.49 & - & - & 31.59 & - & - & 9.95 & - & - & 2.49 & - & - \\
        Vanilla (PA) & 31.51 M & 22.31 & -0.63 & 9.87 & 89.97 & -3.17 & 49.78 & 15.89 & -0.79 & 5.57 & 17.10 & 5.30 & 7.27 \\
        LwF~\citep{li2017learning} & 31.51 M & 24.39 & 0.56 & 9.79 & 90.34 & -1.39 & 44.87 & 15.27 & -0.24 & 4.25 & 6.13 & -2.17 & 4.82 \\
        EWC~\citep{lee2017overcoming} & 31.51 M & 22.92 & -0.87 & 10.14 & 86.73 & -5.14 & 45.00 & 14.93 & -0.51 & 4.48 & 21.59 & 9.83 & 8.81 \\
        MoE-LoRA~\citep{liu2024moe} & 208.05 M & 17.33 & -2.09 & 5.55 & 71.13 & -2.85 & 28.66 & 13.21 & -0.32 & 1.62 & 8.04 & -3.33 & 6.02 \\
        ModalPrompt~\citep{zeng2025modalprompt} & 19.96 M & 11.70 & -2.07 & 1.99 & 36.32 & -8.64 & 8.10 & 9.88 & -0.69 & 0.15 & 4.06 & -2.62 & 4.26 \\
        \midrule
        \textbf{ECA (Ours)} & 31.51 M & \textbf{26.88} & \textbf{0.73} & \textbf{11.82} & \textbf{94.49} & \textbf{-1.37} & \textbf{49.57} & \textbf{16.10} & \textbf{-0.08} & \textbf{4.98} & \textbf{22.61} & \textbf{10.21} & \textbf{9.27} \\
        \bottomrule
    \end{tabular}
    \end{adjustbox}
    \vspace{-0.15in}
\end{table*}
\section{ECA on projector-based Multi-modal LLMs}
\label{app:eca_llava}
In this section, we further illustrate how to instantiate our proposed ECA on a projector-based multi-modal LLM (MLLM), \eg LLaVA~\cite{liu2023visual}. Unlike the Q-Former in BLIP-2, which uses learnable query tokens to directly interact with visual embeddings, LLaVA adopts a visual \emph{projector} that maps visual features into the language token space, and then relies on self-attention in the LLM to achieve visual–language alignment. In this case, we regard the top $L$ layers of the LLM as the effective alignment module. Based on this view, Fig.~\ref{fig:llava_pipeline} illustrates how we instantiate ECA on a projector-based MLLM.
As shown in Fig.~\ref{fig:llava_pipeline}, a frozen visual encoder followed by a pre-trained projector produces visual tokens. These visual tokens enter the Mixture-of-Query (MoQ) module, which learns task-specific soft prompts and attentively aggregates them. The visual tokens, aggregated soft prompts, and textual prompts are concatenated and fed into the LLM, whose top $L$ layers are equipped with FeDEx. FeDEx selectively expands parallel adapters in these layers based on the FIM-based conflict score, so that new features are incorporated while preserving established alignment. Meanwhile, Dictionary Replay (DR) maintains an embedding dictionary and replays it during training to retain information from previous tasks.

For DR on projector-based MLLMs, we use the dictionary atoms as visual tokens and pass them, together with soft prompts generated from MoQ, and a textual prompt, into the LLM. Concretely, we first use the model trained by previous tasks (teacher) to generate pseudo captions conditioned on the dictionary atoms, soft prompts, and textual prompt. We then feed the same inputs into the current training model (student) and compute a token-level KL divergence between the teacher and student predictive distributions as a knowledge distillation loss. This allows the embedding dictionary to replay past visual semantics without storing raw images, encouraging the projector-based MLLM with ECA to preserve alignment learned from earlier tasks.

Additionally, we instantiate ECA and representative baselines on a projector-based MLLM, LLaVA-v0~\cite{liu2023visual}. This experiment aims to verify whether our continual alignment framework generalizes beyond Q-Former based VLMs. It is not intended to establish state-of-the-art performance on the latest MLLM backbones. We use LLaVA-v0 because its training recipe is publicly documented, which enables a controlled evaluation where ToS-TextCaps and ToS-TextVQA are outside its published pre-training distribution. LLaVA-v0 uses CLIP-L~\cite{radford2021learning} as the visual encoder and Vicuna-7B-v0~\cite{vicuna2023} as the base LLM. We adapt Vanilla (PA), EWC, MoE-LoRA, and ECA to LLaVA-v0 and evaluate them on ToS-TextCaps and ToS-TextVQA.

We follow the system prompting template of LLaVA-v0 and develop specific human prompts for different datasets. For ToS-TextCaps, we use the prompt \textit{``Based on OCR: {}. Describe this image in detail with one sentence.''} For ToS-TextVQA, we use the prompt \textit{``Based on OCR: {}. {} Answer in one word.''} where the second placeholder corresponds to the question text. To instantiate continual alignment on this architecture, we treat the top $L{=}6$ layers of the base LLM as the alignment module and adapt all methods on these layers. We set the PA low-rank to $128$ for attention layers and $512$ for feed-forward layers, and we set all PA scales to $4$ following~\citep{he2022towards}. For MoE-LoRA, we use $16$ experts per layer and insert MoE-LoRA into every feed-forward network. Each expert uses a low-rank of $128$ due to our GPU memory budget. The remaining settings follow Appendix~\ref{app:ep_detail}. For ModalPrompt~\cite{zeng2025modalprompt}, we follow the original paper, set the prompt number as $10$ per task, choose top-$3$ prompt sets for training and inference, and we train $10$ epochs for each task.

We report baselines that can be instantiated on LLaVA-v0 under the same alignment-module scope with minimal additional design choices, including PA-style adapters, LwF, EWC, ModalPrompt, and MoE-LoRA. Prompt-pool methods such as Dual-Prompt and CODA-Prompt are not inherently tied to Q-Former. In our BLIP-2 experiments, we instantiated them by inserting deep prompts into the self-attention layers of the Q-Former. Porting these methods to projector-based MLLMs requires design choices on prompt injection layers, prompt pooling, and routing inside the LLM. These choices introduce additional hyperparameters and confounding factors, so we leave a thorough port of prompt-pool baselines to projector-based MLLMs as future work. However, we here include the recent prompt-pool method, Modal Prompt, which is originally designed with LLaVA, and it performs better than other prompt-pool methods. Thus, we also compare with the Modal Prompt to show our great performance.

As shown in Tab.~\ref{tab:llava_text}, ECA achieves the best final average performance on both ToS-TextCaps and ToS-TextVQA. It also improves BWT compared with Vanilla (PA) and MoE-LoRA while using far fewer trainable parameters than MoE-LoRA. These results indicate that ECA remains effective on projector-based MLLMs and can mitigate catastrophic forgetting beyond our BLIP-2 instantiation. We also observe that the absolute performance of LLaVA-v0 on these two benchmarks is lower than that of our BLIP-2 instantiation. We attribute this gap to two factors. First, LLaVA-v0 is primarily tuned for instruction following that often favors longer responses, whereas ToS-TextCaps and ToS-TextVQA emphasize short outputs. Second, LLaVA-v0 relies on CLIP features with a simple projector, while BLIP-2 uses a Q-Former that more directly injects fine-grained visual information into learnable tokens. Since OCR-based OpenITG tasks depend on detailed visual cues, BLIP-2 attains stronger performance in this setting. Nevertheless, our ECA achieves the strongest performance among all baselines on this projector-based MLLM, including ModalPrompt, which supports the generality of our framework beyond Q-Former-based VLMs. ModalPrompt underperforms in our setting because it is designed for continual instruction tuning under evolving textual instructions. In contrast, our ToS setting involves continual shifts in overlapping visual semantic topics.

\section{Case Study}\label{app:case_study}
\begin{figure*}[!tbh]
\centering
\begin{subfigure}{0.49\linewidth}
    \centering
    \includegraphics[width=\linewidth]{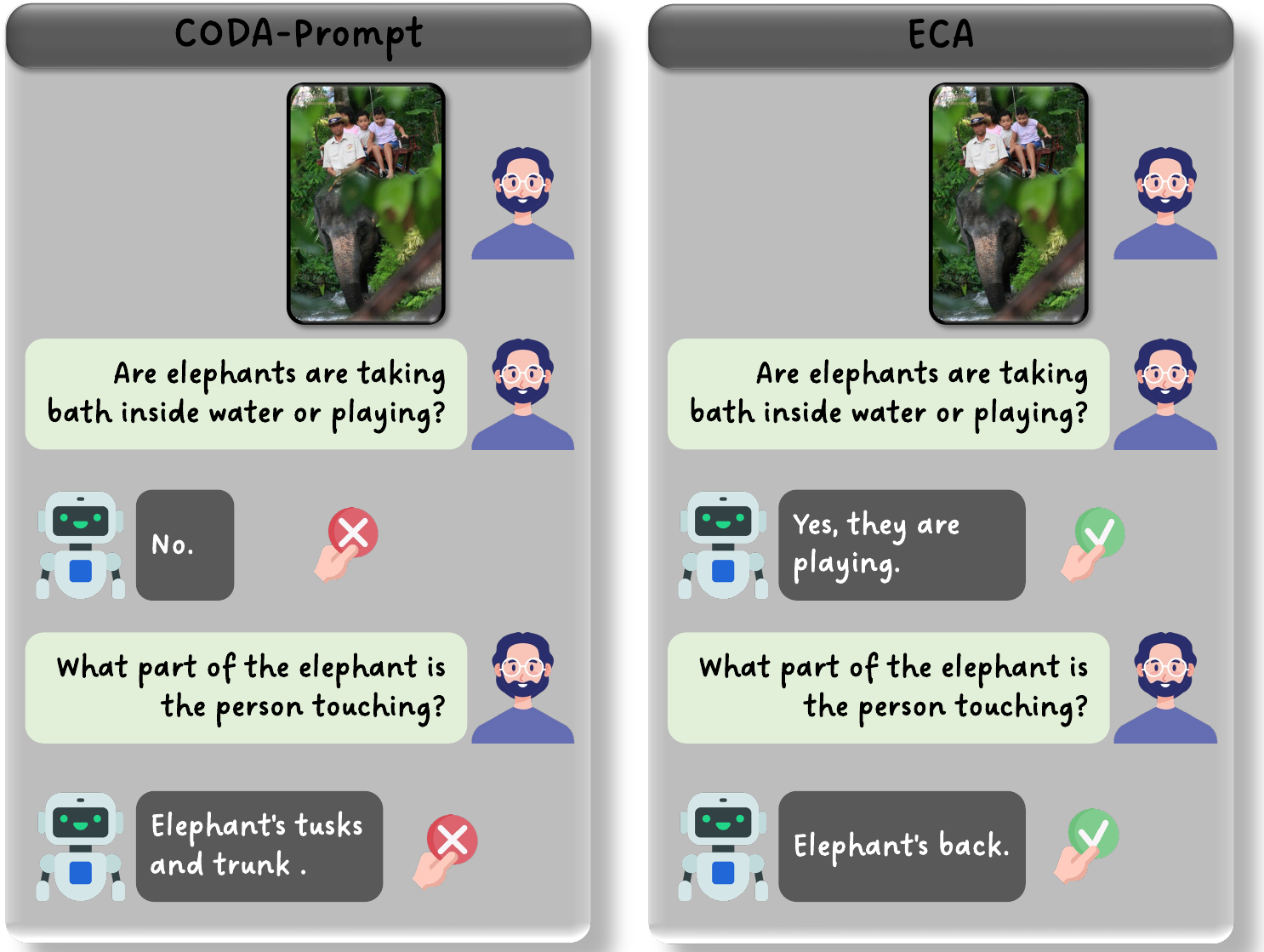}
    \label{fig:case_1}
\end{subfigure}
\hfill
\begin{subfigure}{0.49\linewidth}
    \centering
    \includegraphics[width=\linewidth]{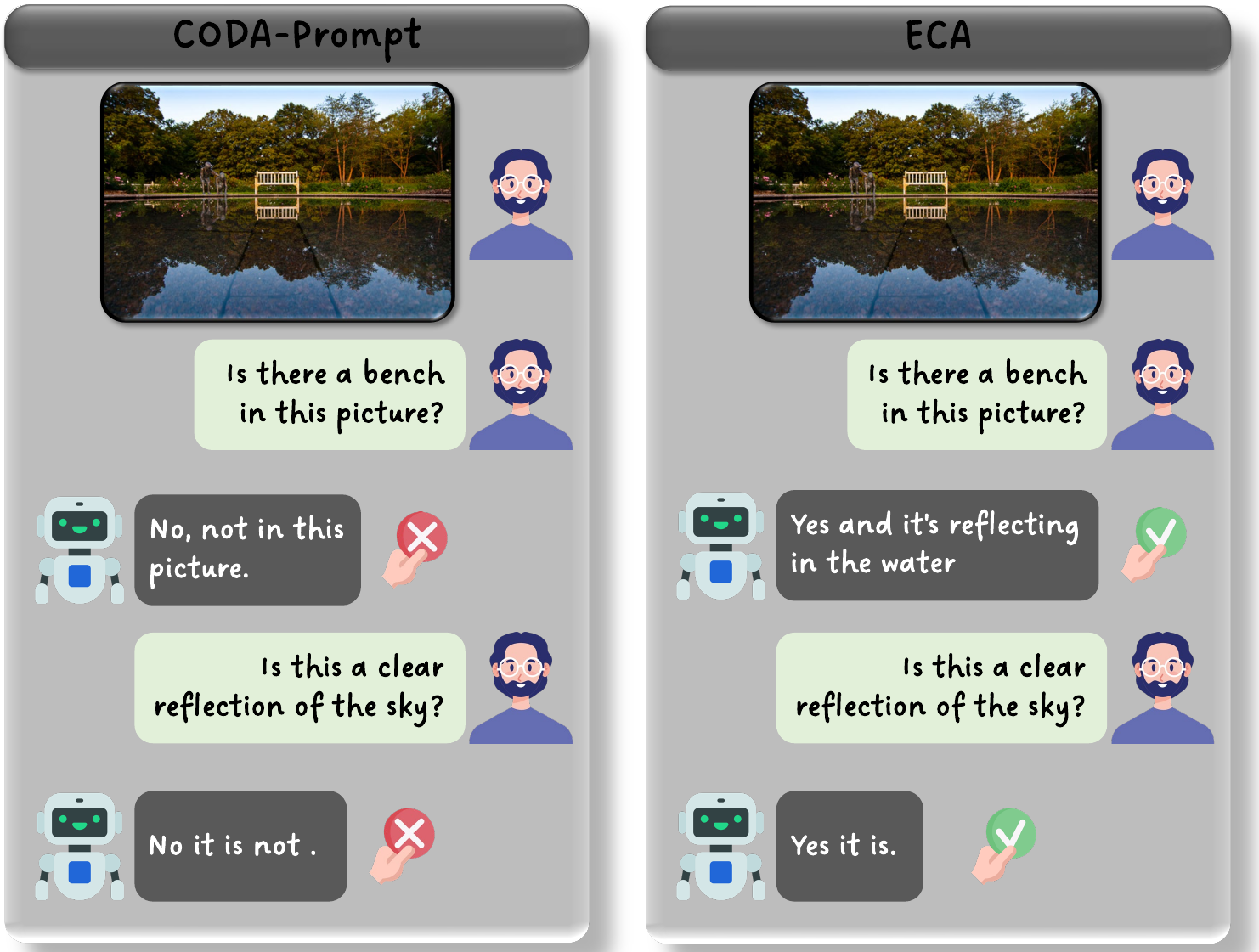}
    \label{fig:case_2}
\end{subfigure}
\begin{subfigure}{0.49\linewidth}
    \centering
    \includegraphics[width=\linewidth]{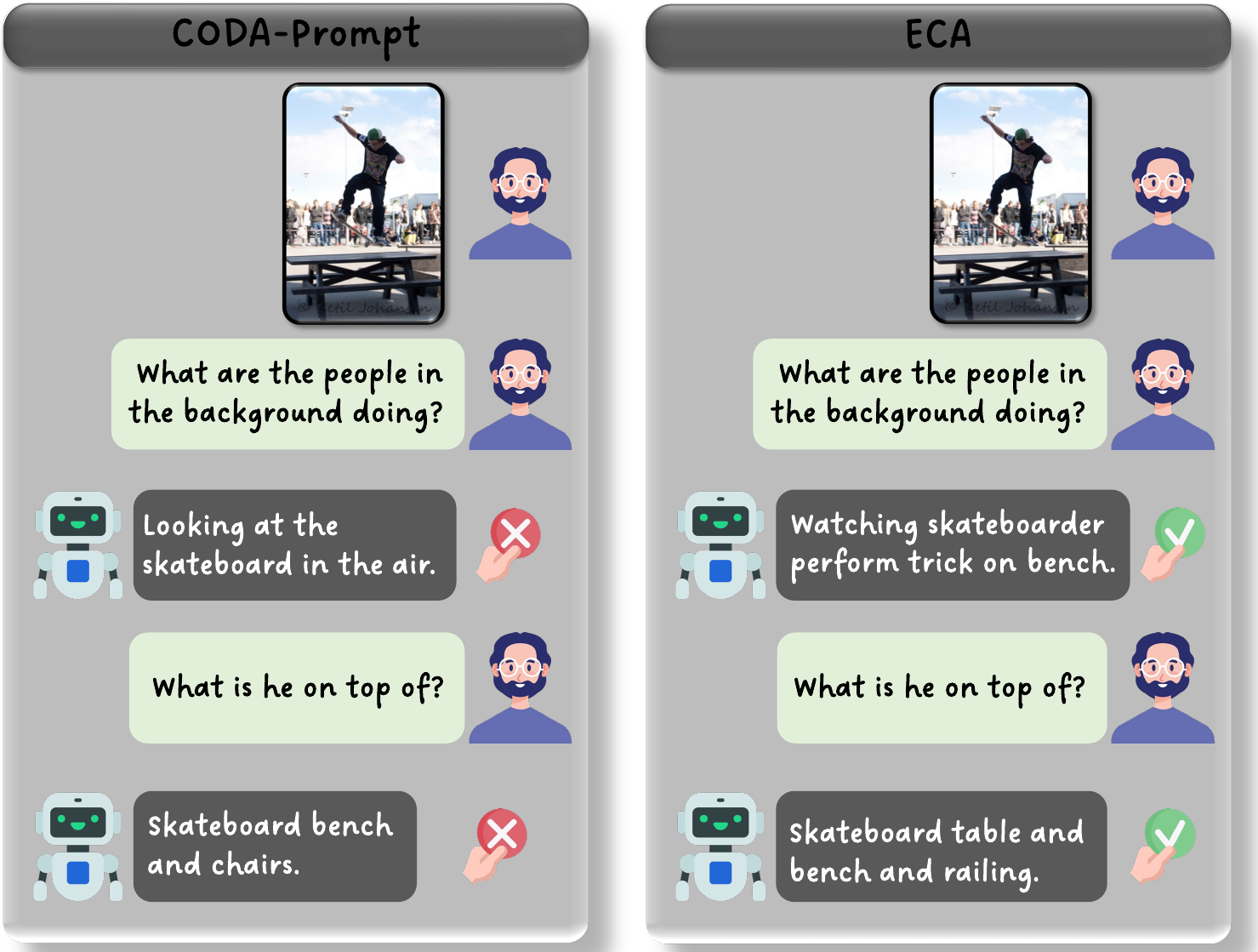}
    \label{fig:case_3}
\end{subfigure}
\begin{subfigure}{0.49\linewidth}
    \centering
    \includegraphics[width=\linewidth]{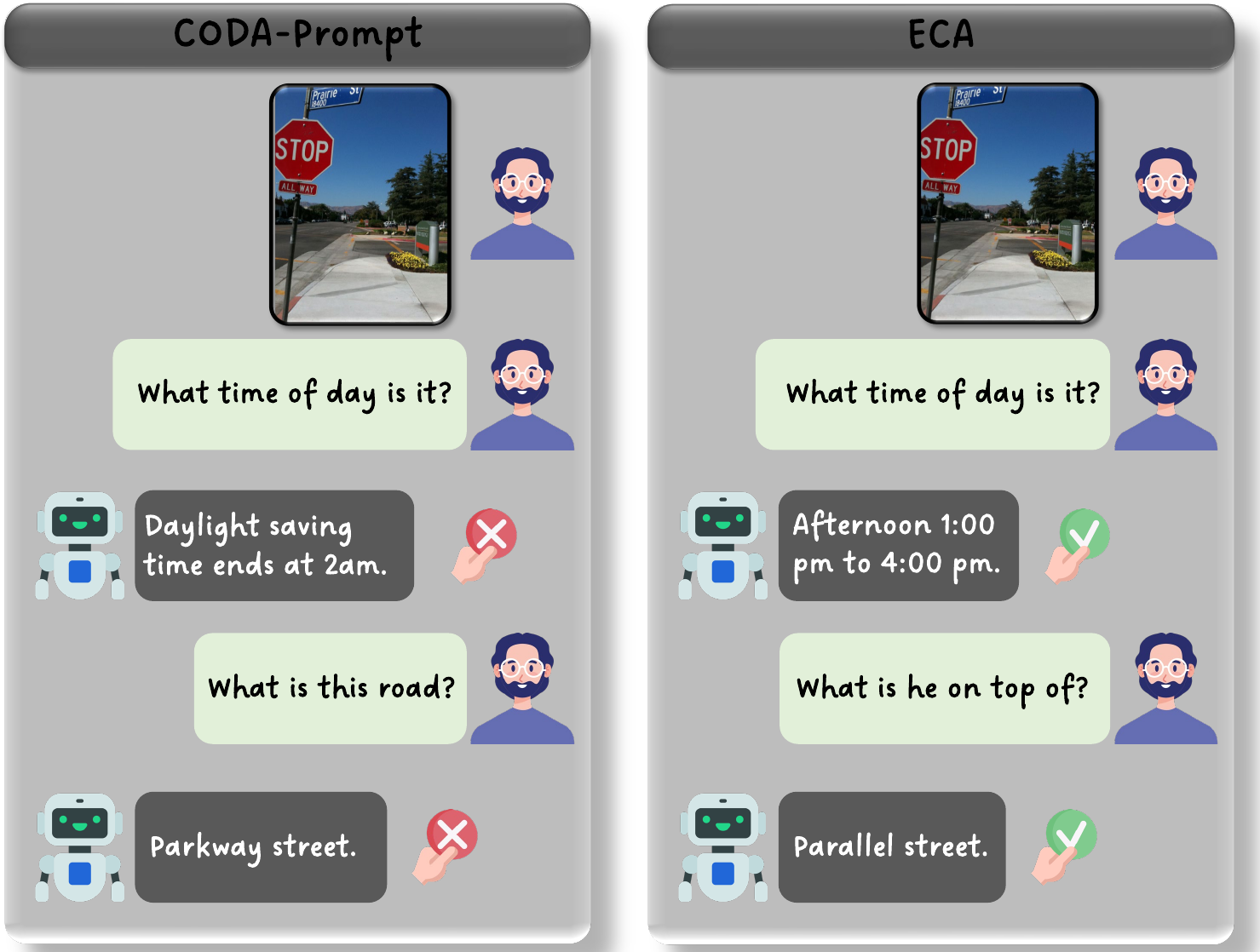}
    \label{fig:case_4}
\end{subfigure}
\caption{Comparison on the case from the first task in ToS-VQAv2.}
\label{fig:case_study}
\end{figure*}

\begin{figure*}[!t]
\centering
\begin{subfigure}{0.49\linewidth}
    \centering
    \includegraphics[width=\linewidth]{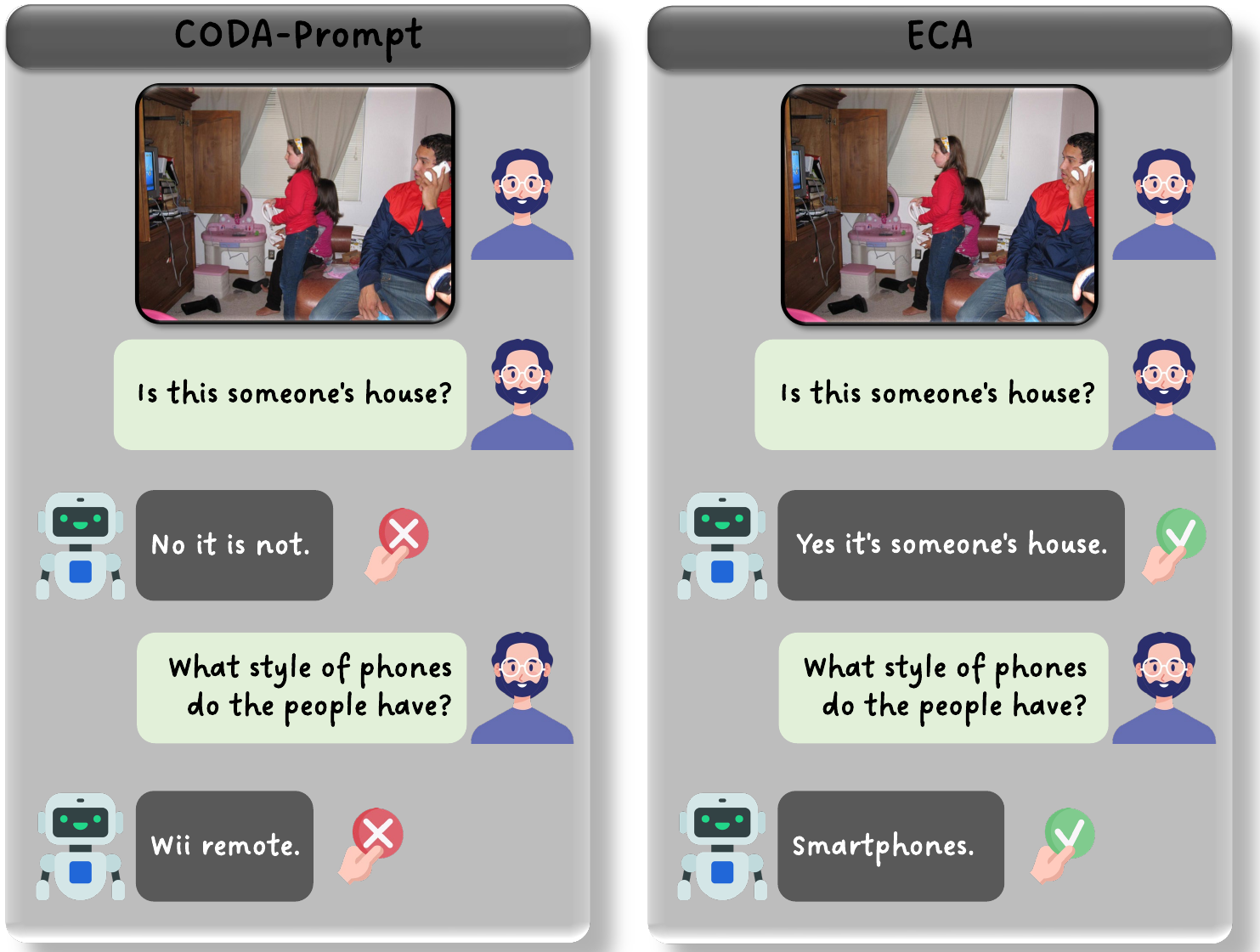}
    \label{fig:case_5}
\end{subfigure}
\hfill
\begin{subfigure}{0.49\linewidth}
    \centering
    \includegraphics[width=\linewidth]{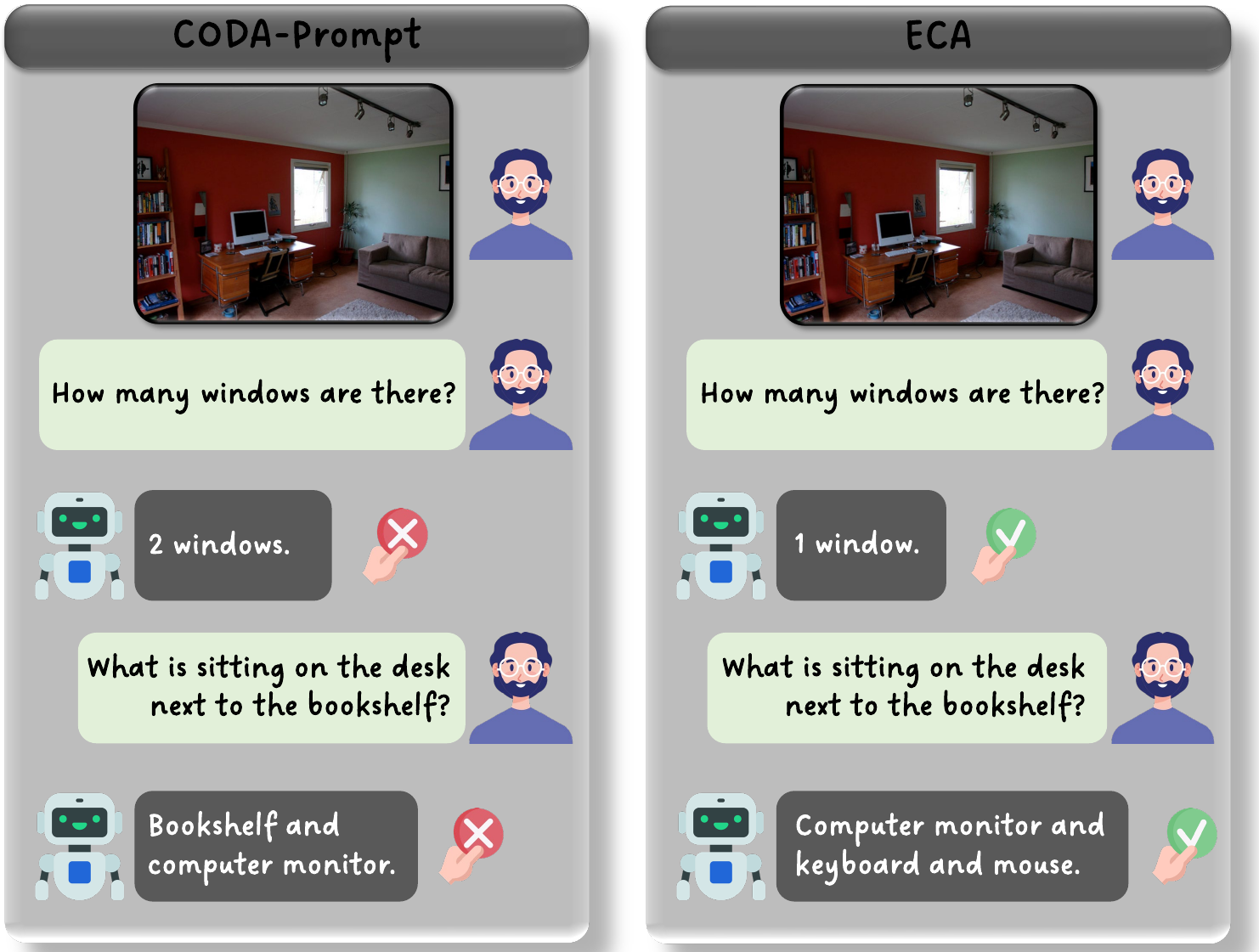}
    \label{fig:case_6}
\end{subfigure}
\begin{subfigure}{0.49\linewidth}
    \centering
    \includegraphics[width=\linewidth]{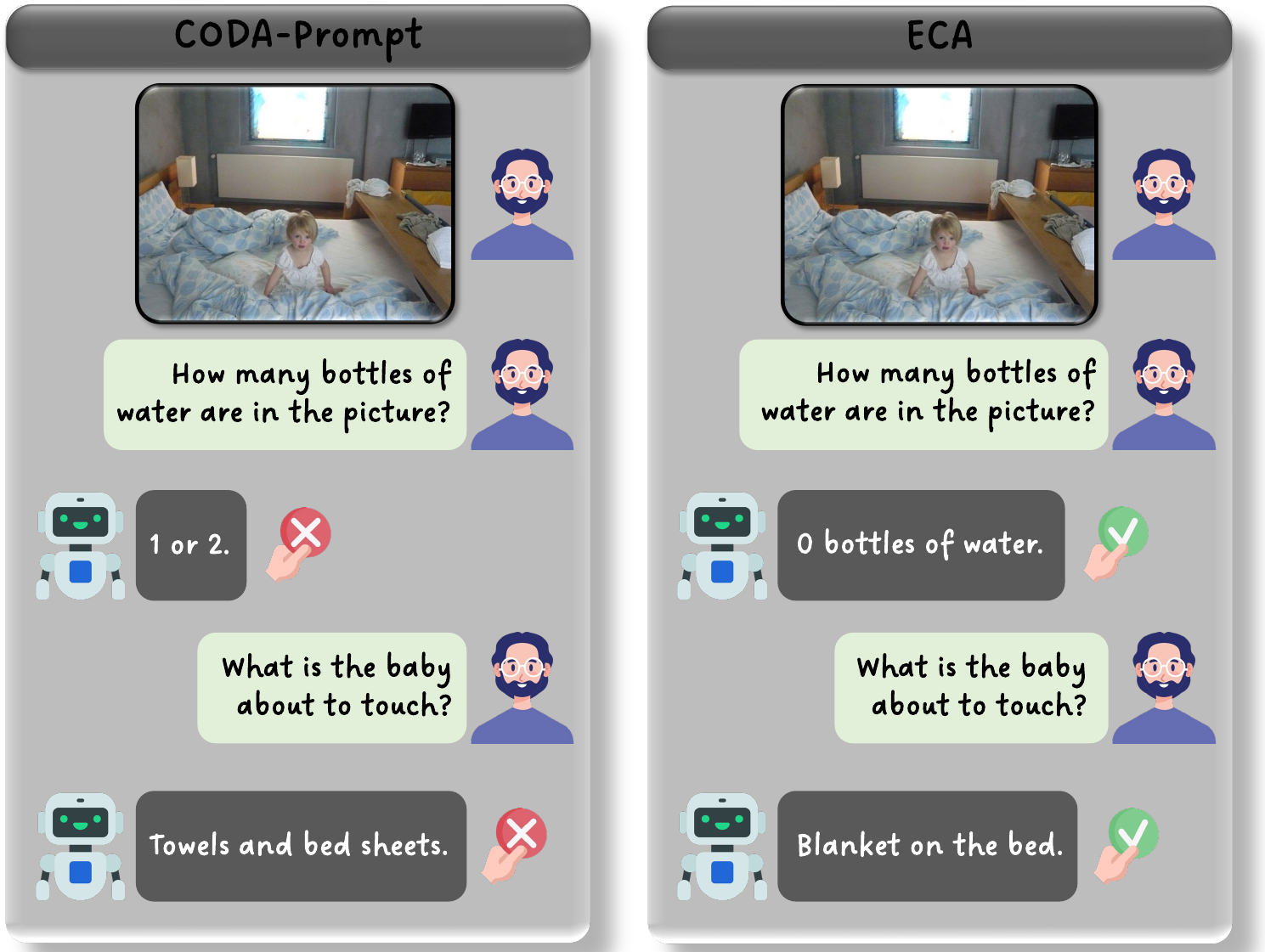}
    \label{fig:case_7}
\end{subfigure}
\begin{subfigure}{0.49\linewidth}
    \centering
    \includegraphics[width=\linewidth]{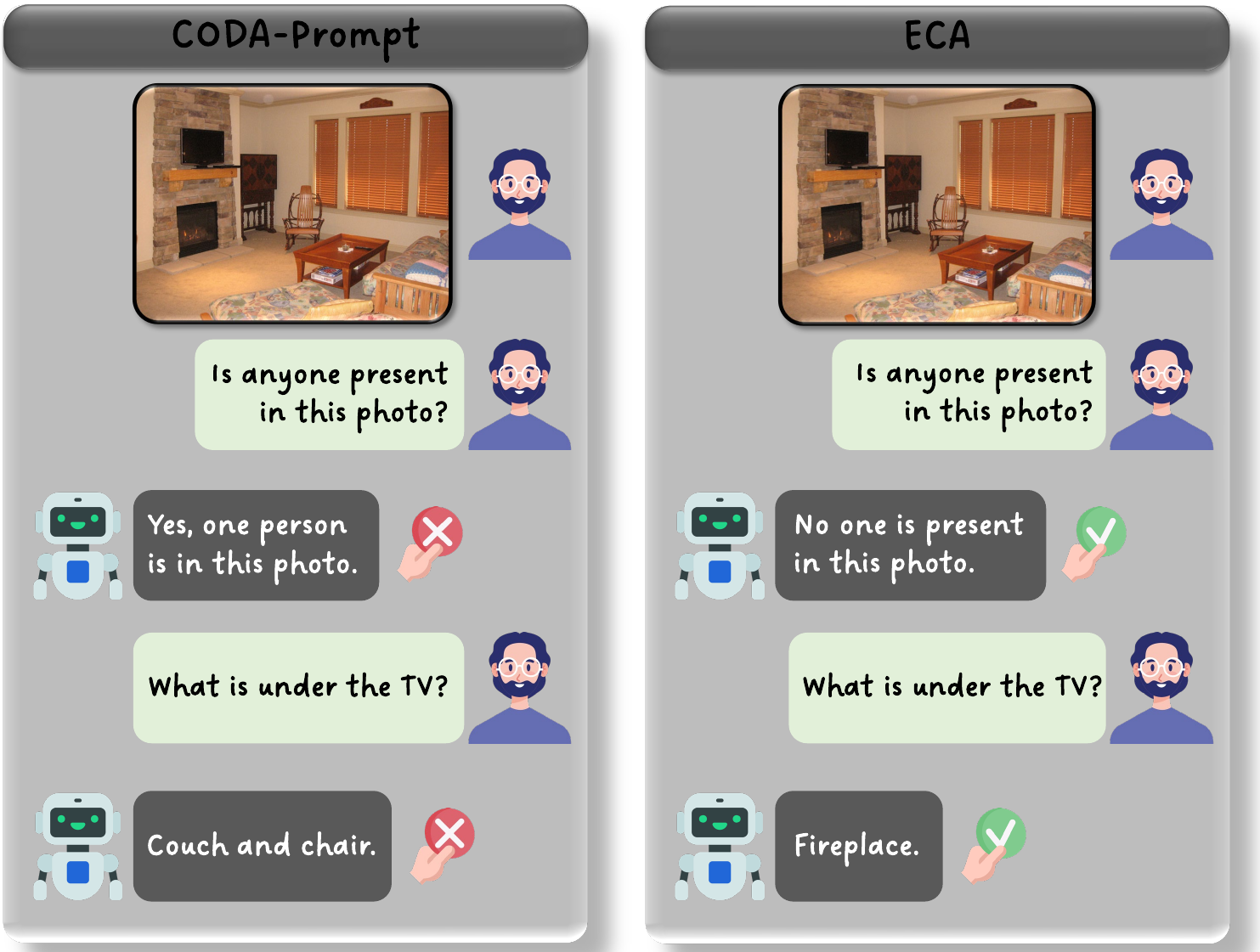}
    \label{fig:case_8}
\end{subfigure}
\caption{Comparison on the case from the middle task in ToS-VQAv2.}
\label{fig:case_study_2}
\end{figure*}

In this section, we show the comparison on real cases between our ECA and the CODA-Prompt, which is the second-best uni-modal exemplar-free baseline on the ToS-VQAv2. As shown in Fig.~\ref{fig:case_study}, we use real cases of the first main topic after training BLIP-2 by both methods on all tasks from ToS-VQAv2 to test the model. The CODA-Prompt ruined the alignment established from previous tasks while training the model sequentially. However, our ECA aligns both modalities continually, and finally, it can still respond in great shape to the previously learned knowledge. The reason is that the CODA-Prompt only uses soft prompts to learn the new knowledge, which is limited. Thus, it is hard to achieve continual alignment for VLM only by the prompt when the data has complex semantics, and tasks have overlapping semantics. We also compare cases from the middle task, which is closer to the final task. As shown in Fig.~\ref{fig:case_study_2}, the CODA-Prompt still can not understand well on some easy cases, while our ECA still shows a great performance on mitigating catastrophic forgetting.

\end{document}